\documentclass[11pt]{article}

\usepackage{microtype}
\usepackage{graphicx}
\usepackage{subfigure}
\usepackage{booktabs} 

\usepackage{hyperref}

\usepackage{amsmath}
\usepackage{amssymb, bm}
\usepackage{amsthm}
\usepackage{float}
\usepackage{color}
\usepackage{subfigure}
\usepackage{rotating}
\usepackage{url}
\usepackage{multirow}
\usepackage{algorithm}
\usepackage{algorithmic}
\usepackage{graphicx} 
\usepackage{subfigure} 
\usepackage{url}
\usepackage{wrapfig}

\newtheorem{prop}{\textbf{Proposition}}

\newtheorem{theorem}{\textbf{Theorem}}
\newtheorem{lemma}{\textbf{Lemma}}
\newtheorem{corollary}{\textbf{Corollary}}

\newcommand{\myfont}{\fontsize{8pt}{\baselineskip}\selectfont}

\begin{document}

\title{Candidates vs. Noises Estimation for Large Multi-Class Classification Problem}

\author{Lei Han, Yiheng Huang, Tong Zhang\\
Tencent AI Lab, Shenzhen, China \\
}

\date{}

\maketitle

\begin{abstract}
This paper proposes a method for multi-class classification problems, where the number of classes $K$ is large. The method, referred to as {\em Candidates vs. Noises Estimation} (CANE), selects a small subset of candidate classes and samples the remaining classes. We show that CANE is always consistent and computationally efficient. Moreover, the resulting estimator has low statistical variance approaching that of the maximum likelihood estimator, when the observed label belongs to the selected candidates with high probability. In practice, we use a tree structure with leaves as classes to promote fast beam search for candidate selection. We further apply the CANE method to estimate word probabilities in learning large neural language models. Extensive experimental results show that CANE achieves better prediction accuracy over the Noise-Contrastive Estimation (NCE), its variants and a number of the state-of-the-art tree classifiers, while it gains significant speedup compared to standard $\mathcal{O}(K)$ methods.
\end{abstract}

\section{Introduction}
In practice one often encounters multi-class classification problem with a large number of classes. For example, applications in image classification \cite{russakovsky2015imagenet} and language modeling \cite{mikolov2010recurrent} usually have tens to hundreds of thousands of classes. Under such cases, training the standard softmax logistic or one-against-all models becomes impractical.

One promising way to handle the large class size is to use sampling. In language models, a commonly adopted technique is {\em Noise-Contrastive Estimation} (NCE) \cite{gutmann2010noise}. This method is originally proposed for estimating probability densities and has been applied to various language modeling situations, such as learning word embeddings, context generation and neural machine translation \cite{mnih2012fast,mnih2013learning,vaswani2013decoding,sordoni2015neural}. NCE reduces the problem of multi-class classification to binary classification problem, which discriminates between a target class distribution and a noise distribution and a few noise classes are sampled as a representation of the entire noise space. In general, the noise distribution is given a priori. For example, a power-raised unigram distribution has been shown to be effective in language models \cite{mikolov2013distributed,ji2015blackout,mnih2012fast}. Recently, some variants of NCE have been proposed. The \emph{Negative Sampling} \cite{mikolov2013distributed} is a simplified version of NCE that ignores the numerical probabilities in the distributions and discriminates between only the target class and noise samples; the One vs. Each \cite{Titsias2016one} solves a very similar problem motivated by bounding the softmax logistic log-likelihood. Two other variants, {\em BlackOut} \cite{ji2015blackout} and {\em complementary sum sampling} \cite{botev2017complementary}, employ parametric forms of the noise distribution and use sampled noises to approximate the normalization factor. In summary, NCE and its variants use (only) the observed class versus the noises; by sampling the noises, these methods avoid the costly computation of the normalization factor to achieve fast training speed. In this paper, we will generalize the idea by using a subset of classes (which can be automatically learned), called candidate classes, against the remaining noise classes. Compared to NCE, this approach can significantly improve the statistical efficiency when the true class belongs to the candidate classes with high probability.

Another type of popular methods for large class space is the tree structured classifier \cite{beygelzimer2009error,bengio2010label,deng2011fast,choromanska2015logarithmic,daume2016logarithmic,jernite2016simultaneous}. In these methods, a tree structure is defined over the classes which are treated as leaves. Each internal node of the tree is assigned with a local classifier, routing the examples to one of its descendants. Decisions are made from the root until reaching a leaf. Then, the multi-class classification problem is reduced to solving a number of small local models defined by a tree, which typically admits a logarithmic complexity on the total number of classes. Generally, tree classifiers gain training and prediction speed while suffering a loss of accuracy. The performance of tree classifier may rely heavily on the quality of the tree \cite{mnih2009scalable}. Earlier approaches use fixed tree, such as the Filter Tree \cite{beygelzimer2009error} and the Hierarchical Softmax (HSM) \cite{morin2005hierarchical}. Recent methods are able to adjust the tree and learn the local classifiers simultaneously, such as the LOMTree \cite{choromanska2015logarithmic} and Recall Tree \cite{daume2016logarithmic}. Our approach is complementary to these tree classifiers, because we study the orthogonal issue of consistent class sampling, which in principle can be combined with many of these tree methods. In fact, a tree structure will be used in our approach to select a small subset of candidate classes. Since we focus on the class sampling aspect, we do not necessarily employ the best tree construction method in our experiments.

In this paper, we propose a method to efficiently deal with the large class problem by paying attention to a small subset of candidate classes instead of the entire class space. Given a data point $\bm{x}$ (without observing $y$), we select a small number of competitive candidates as a set $\mathcal{C}_{\bm{x}}$. Then, we sample the remaining classes, which are treated as noises, to represent the entire noise space in the large normalization factor. The estimation is referred to as {\em Candidates vs. Noises Estimation} (CANE). 
We show that CANE is consistent and its computation using stochastic gradient method is independent of the class size $K$. Moreover, the statistical variance of the CANE estimator can approach that of the maximum likelihood estimator (MLE) of the softmax logistic regression when $\mathcal{C}_{\bm{x}}$ can cover the target class $y$ with high probability. This statistical efficiency is a key advantage of CANE over NCE, and its effect can be observed in practice.

We then describe two concrete algorithms: the first one is a generic stochastic optimization procedure for CANE; the second one employs a tree structure with leaves as classes to enable fast beam search for candidate selection. We also apply CANE to solve the word probability estimation problem in neural language modeling. Experimental results conducted on both classification and neural language modeling problems show that CANE achieves significant speedup compared to the standard softmax logistic regression. Moreover, it achieves superior performance over NCE, its variants, and a number of the state-of-the-art tree \mbox{classifiers}. 

\section{Candidates vs. Noises Estimation}
Consider a $K$-class classification problem ($K$ is large) with $n$ training examples $(\bm{x}_i, y_i)|_{i=1}^n$, where $\bm{x}_i$ is from an input space $\mathcal{X}$ and $y_i\in\{1,\cdots,K\}$. The softmax logistic regression solves
\begin{align}
	\max_{\bm{\theta}}
	\frac{1}{n}\sum_{i=1}^n\sum_{k=1}^K
	\mathbb{I}(y_i=k)
	\log\frac{e^{s_k(\bm{x}_i, \bm{\theta})} }{\sum_{k'=1}^K e^{s_{k'}(\bm{x}_i,\bm{\theta}) } },
	\label{eq:softmax}
\end{align}
where $s_k(\bm{x},\bm{\theta})$ for $k=1,\cdots,K$ is a model parameterized by $\bm{\theta}$. Solving Eq. (\ref{eq:softmax}) requires computing a score for every class and the summation in the normalization factor, which is very expensive when $K$ is large.

Generally speaking, given $\bm{x}$, only a small number of classes in the entire class space might be competitive to the true class. Therefore, we propose to find a small subset of classes as a candidate set $\mathcal{C}_{\bm{x}}\subset\{1,\cdots,K\}$ and treat the classes outside $\mathcal{C}_{\bm{x}}$ as noises, so that we can focus on the small set $\mathcal{C}_{\bm{x}}$ instead of the entire $K$ classes. We will discuss one way to choose $\mathcal{C}_{\bm{x}}$ in Section \ref{sec:algo}. Denote the remaining $K-|\mathcal{C}_{\bm{x}}|$ noises as a set $\mathcal{N}_{\bm{x}}$, so $\mathcal{N}_{\bm{x}}$ is the complementary set of $\mathcal{C}_{\bm{x}}$. We propose to sample some noise class $j\in \mathcal{N}_{\bm{x}}$ to represent the entire $\mathcal{N}_{\bm{x}}$. That is, we replace the partial summation $\sum_{j\in\mathcal{N}_{\bm{x}}}e^{s_j(\bm{x},\bm{\theta})}$ in the denominator of Eq. (\ref{eq:softmax}) by $e^{s_j(\bm{x},\bm{\theta})}/q_{\bm{x}}(j)$ using some sampled class $j$ with an {\em arbitrary sampling probability} $q_{\bm{x}}(j)$, where $q_{\bm{x}}(j)\in(0,1)$ and $\sum_{j\in\mathcal{N}_{\bm{x}}}q_{\bm{x}}(j)=1$. Thus, the denominator $\sum_{k'=1}^K e^{s_{k'}(\bm{x},\bm{\theta})}$ will be approximated as $\sum_{k'\in \mathcal{C}_{\bm{x}}}e^{s_{k'}(\bm{x},\bm{\theta})}+e^{s_j(\bm{x},\bm{\theta})}/q_{\bm{x}}(j)$.
Given example $(\bm{x},y)$ and its candidate set $\mathcal{C}_{\bm{x}}$, if $y\in\mathcal{C}_{\bm{x}}$, then for some sampled noise class $j$, we will focus on maximizing the approximated probability
\begin{align}
\frac{e^{s_y(\bm{x},\bm{\theta})}}{
	\sum_{k'\in \mathcal{C}_{\bm{x}}}e^{s_{k'}(\bm{x},\bm{\theta})}+e^{s_j(\bm{x},\bm{\theta})}/q_{\bm{x}}(j)
	};
\label{eq:Pin}
\end{align} 
otherwise, if $y\not\in\mathcal{C}_{\bm{x}}$, we maximize
\begin{align}
\frac{e^{s_y(\bm{x},\bm{\theta})}}
	{\sum_{k'\in \mathcal{C}_{\bm{x}}}e^{s_{k'}(\bm{x},\bm{\theta})}+{e^{s_y(\bm{x},\bm{\theta})}}/{q_{\bm{x}}(y)}}
\label{eq:Pout}
\end{align}
alternatively, where $y$ is treated as the sampled noise in place. Now, with Eqs. (\ref{eq:Pin}) and (\ref{eq:Pout}), in expectation, we will need to solve the following objective:
\begin{align}
\text{maximize}\ R(\bm{\theta})=&\ 
\mathbb{E}_{\bm{x}} 
\Bigg[
\sum_{k\in\mathcal{C}_{\bm{x}}}p(y=k|\bm{x})
\sum_{j\in \mathcal{N}_{\bm{x}}}q_{\bm{x}}(j)
\log\frac{e^{s_k(\bm{x},\bm{\theta})}}{
\sum_{k'\in \mathcal{C}_{\bm{x}}}e^{s_{k'}(\bm{x},\bm{\theta})}\!+\!
\frac{e^{s_j(\bm{x},\bm{\theta})}}{q_{\bm{x}}(j)}
}
\nonumber
\\ 
&+\sum_{k\in\mathcal{N}_{\bm{x}}}p(y=k|\bm{x})
\log\frac{e^{s_k(\bm{x},\bm{\theta})}}
{
\sum_{k'\in \mathcal{C}_{\bm{x}}}e^{s_{k'}(\bm{x},\bm{\theta})}+
\frac{e^{s_k(\bm{x},\bm{\theta})}}{q_{\bm{x}}(k)}
} \Bigg],
\label{eq:Exp}
\end{align}
and empirically, we will need to solve
\begin{align}
\text{maximize}\ 
\hat{R}_n(\bm{\theta})=&\ 
\frac{1}{n}\sum_{i=1}^n
\Bigg[ \mathbb{I}(y_i\in \mathcal{C}_{\bm{x}_i})
\sum_{j\in \mathcal{N}_{\bm{x}_i}}q_{\bm{x}_i}(j)
\log
\frac{e^{s_{y_i}(\bm{x}_i,\bm{\theta})}}
{
\sum_{k'\in \mathcal{C}_{\bm{x}_i}}e^{s_{k'}(\bm{x}_i,\bm{\theta})}\!+\!
\frac{e^{s_j(\bm{x}_i,\bm{\theta})}}{q_{\bm{x}_i}(j)}
}
\nonumber
\\
&+\mathbb{I}(y_i \notin \mathcal{C}_{\bm{x}_i})
\log\frac
{e^{s_{y_i}(\bm{x}_i,\bm{\theta})}}
{
\sum_{k'\in \mathcal{C}_{\bm{x}_i}}e^{s_{k'}(\bm{x}_i,\bm{\theta})}+
\frac{e^{s_{y_i}(\bm{x}_i,\bm{\theta})}}{q_{\bm{x}_i}(y_i)}
}
\Bigg].
\label{eq:obj}
\end{align}
Eq. (\ref{eq:obj}) consists of two summations over both the data points and the classes in the noise set $\mathcal{N}_{\bm{x}}$. Therefore, we can employ a `doubly' stochastic gradient optimization method by sampling both data points $i\in \{1,\ldots,n\}$ and noise classes $j \in \mathcal{N}_{\bm{x}_i}$. 
It is not difficult to check that each stochastic gradient is bounded under reasonable conditions, which
means that the computational cost for solving \eqref{eq:obj} using stochastic gradient is independent of the class number $K$.
Since we only choose a small number of candidates in $\mathcal{C}_{\bm{x}}$, the computation for each stochastic gradient in Eq. (\ref{eq:obj}) is efficient. The above method is referred to as Candidates vs. Noises Estimation (CANE). 

\section{Properties}
In this section, we investigate the statistical properties of CANE. The parameter space of the softmax logistic model in Eq. (\ref{eq:softmax}) has redundancy, observing that adding any  function $h(\bm{x})$ to $s_k(\bm{x},\bm{\theta})$ for $k=1,\cdots,K$ will not change the objective. Similar situation happens for Eqs. (\ref{eq:Exp}) and (\ref{eq:obj}). To avoid this redundancy, one can add some constraints on the $K$ scores or simply fix one of them as zero, e.g., let $s_K(\bm{x},\bm{\theta})=0$. To facilitate the analysis, we will fix $s_K(\bm{x},\bm{\theta})=0$ and consider $\mathcal{C}_{\bm{x}}\cup\mathcal{N}_{\bm{x}}=\{1,\cdots,K-1\}$ within this section. First, we have the following result.
\begin{theorem}[\emph{Infinity-Sample Consistency}]\label{theo:NP}
By viewing the objective $R$ as a function of $\{s_1,\cdots,s_{K-1}\}$, $R$ achieves its maximum if and only if $s_k=\log \frac{p(y=k|\bm{x})}{p(y=K|\bm{x})}$ for $k=1,\cdots,K-1$.
\end{theorem}
In Theorem \ref{theo:NP}, the global optima is exactly the log-odds function with class $K$ as the reference class. Now, considering the parametric form $s_k(\bm{x},\bm{\theta})$, there exists a true parameter $\bm{\theta}^*$ so that $s_k(\bm{x},\bm{\theta}^*)=\log \frac{p(y=k|\bm{x})}{p(y=K|\bm{x})}$ if the model $s_k(\bm{x},\bm{\theta})$ is correctly specified. The following theorem shows that the CANE estimator $\hat{\bm{\theta}}=\arg\max_{\bm{\theta}}\hat{R}_n(\bm{\theta})$ is consistent with the true parameter $\bm{\theta}^*$.

\begin{theorem}[\emph{Finite-Sample Asymptotic Consistency}]\label{theo:con}
Given $\bm{x}$, denote $\mathcal{C}_{\bm{x}}$ as $\{i_1,\cdots,i_{|\mathcal{C}_{\bm{x}}|}\}$ and $\mathcal{N}_{\bm{x}}$ as $\{j_1,\cdots,j_{|\mathcal{N}_{\bm{x}}|}\}$. 
Suppose that the parameter space is compact and $\forall \bm{\theta}\neq\bm{\theta}^*$ such that $\mathbb{P}_{\mathcal{X}}\left(s_k(\bm{x},\bm{\theta})\neq s_k(\bm{x},\bm{\theta}^*)\right)>0$ for $\bm{x}\sim\mathcal{X}$, $k\neq K$.
Assume $\|\nabla_{\bm{\theta}}s_k(\bm{x},\bm{\theta})\|$, $\|\nabla_{\bm{\theta}}^2 s_k(\bm{x},\bm{\theta})\|$ and $\|\nabla_{\bm{\theta}}^3 s_k(\bm{x},\bm{\theta})\|$ for $k\neq K$ are bounded under some norm $\|\cdot\|$ defined on the parameter space of $\bm{\theta}$. 
Then, as $n\rightarrow\infty$, the estimator $\hat{\bm{\theta}}$ converges to $\bm{\theta}^*$.
\end{theorem}
The above theorem shows that similar to the maximum likelihood estimator of Eq. \eqref{eq:softmax}, the CANE estimator in Eq. \eqref{eq:obj} is also consistent.
Next, we have the asymptotic normality for $\hat{\bm{\theta}}$ as follows.
\begin{theorem}[\emph{Asymptotic Normality}]\label{theo:norm}
Under the same assumption used in Theorem \ref{theo:con}, as $n\rightarrow\infty$, $\sqrt{n}(\hat{\bm{\theta}}-\bm{\theta}^*)$ follows the asymptotic normal distribution:
		\begin{align}
			\sqrt{n}(\hat{\bm{\theta}}-\bm{\theta}^*)\xrightarrow{d}\mathbb{N}(\bm{0},[\mathbb{E}_{\bm{x}}\bm{\nabla}\bm{M}\bm{\nabla}^\top]^{-1}),
		\end{align}
where
\begin{align*}
&\bm{M}=\sum_{j\in\mathcal{N}_{\bm{x}}}q_{\bm{x}}(j)
\Bigg[
diag\left(\mathbf{u}_j\right)-
\frac{1}{p(K,\bm{x})+\sum_{k\in\mathcal{C}_{\bm{x}}}p(k,\bm{x})+
\frac{p(j,\bm{x})}{q_{\bm{x}}(j)}}
\mathbf{u}_j\mathbf{u}_j^\top
\Bigg],
\\
&\mathbf{u}_j=\big(
\underbrace{p(i_1,\bm{x}), \cdots, p(i_{|\mathcal{C}_{\bm{x}}|},\bm{x})}_{\text{The candidate part}}
,
\underbrace{
0, \cdots, {p(j,\bm{x})}/{q_{\bm{x}}(j)}, \cdots, 0}_{\text{The noise part}}
\big)^\top,
\text{ for } j=j_1,\cdots,j_{|\mathcal{N}_{\bm{x}}|},
\\
&\bm{\nabla}=diag\big(\big[\nabla_{\bm{\theta}}s_{i_1}(\bm{x},\bm{\theta}),\cdots,\nabla_{\bm{\theta}}s_{i_{|\mathcal{C}_{\bm{x}}|}}(\bm{x},\bm{\theta}),\nabla_{\bm{\theta}}s_{j_1}(\bm{x},\bm{\theta}),
\cdots,\nabla_{\bm{\theta}}s_{j_{|\mathcal{N}_{\bm{x}}|}}(\bm{x},\bm{\theta})\big]^\top\big).
\end{align*}
\end{theorem}
 
Theorem~\ref{theo:norm} shows that the CANE method has a statistical variance of $[\mathbb{E}_{\bm{x}}\bm{\nabla}\bm{M}\bm{\nabla}^\top]^{-1}$. As we will see in the next corollary, if one can successfully choose the candidate set $\mathcal{C}_{\bm{x}}$ so that it covers the observed label $y$ with high probability, then the difference between the statistical variance of CANE and that of Eq. \eqref{eq:softmax} is small. Therefore, choosing a good candidate set can be important for practical applications. Moreover, under standard conditions, the computation of CANE using stochastic gradient is independent of the class size $K$ because the variance of stochastic gradient is bounded.

\begin{corollary}[\emph{Low Statistical Variance}]\label{theo:IV}
The variance of the maximum likelihood estimator for the softmax logistic regression in Eq. \eqref{eq:softmax} has the form $[\mathbb{E}_{\bm{x}}\bm{\nabla}\bm{M}^{mle}\bm{\nabla}^\top]^{-1}$. If $\sum_{k\in\mathcal{C}_{\bm{x}}\cup\{K\}}p(k,\bm{x})\rightarrow1$, i.e., the probability that $\mathcal{C}_{\bm{x}}\cup\{K\}$ covers the observed class label $y$ approaches 1, then 
\[
[\mathbb{E}_{\bm{x}}\bm{\nabla}\bm{M}\bm{\nabla}^\top]^{-1}\to [\mathbb{E}_{\bm{x}}\bm{\nabla}\bm{M}^{mle}\bm{\nabla}^\top]^{-1} .
\]
\end{corollary}

\section{Algorithm}\label{sec:algo}
In this section, we propose two algorithms. The first one is a general optimization procedure for CANE. The second implementation provides an efficient way to select a competitive set $\mathcal{C}_{\bm{x}}$ using a tree structure defined on the classes.
\subsection{A General Optimization Algorithm}

\begin{algorithm}[t]
\small
\caption{A general optimization procedure for CANE.}
\label{algo:generic}
\begin{algorithmic}[1]
\STATE \textbf{Input:} $K$, $(\bm{x}_i, y_i)|_{i=1}^n$, number of candidates $N_c=|\mathcal{C}_{\bm{x}}|$, number of sampled noises $N_n=|T_{\bm{x}}|$, sampling strategy $\bm{q}$ and learning rate $\eta$.
\STATE \textbf{Output:} $\hat{\bm{\theta}}$.
\vskip 0.1in
\STATE Initialize $\bm{\theta}$;
\FOR {every sampled example}
\STATE Receive example $(\bm{x},y)$;
\STATE Find the candidate set $\mathcal{C}_{\bm{x}}$;
\IF {$y\in\mathcal{C}_{\bm{x}}$}
\STATE Sample $N_n$ noises outside $\mathcal{C}_{\bm{x}}$ according to $\bm{q}$ and denote the selected noise set as $T_{\bm{x}}$;
\STATE $\bm{\theta}\leftarrow\bm{\theta}+\eta\nabla_{\bm{\theta}}\hat{R}$ with
$\nabla_{\bm{\theta}}\hat{R}$ given by
\begin{align}
\nabla_{\bm{\theta}}s_y(\bm{x},\bm{\theta})-\frac{1}{|T_{\bm{x}}|}\sum_{j\in T_{\bm{x}}}
\left[
\frac{
\sum_{k'\in\mathcal{C}_{\bm{x}}}
e^{s_{k'}(\bm{x},\bm{\theta})}\nabla_{\bm{\theta}}s_{k'}(\bm{x},\bm{\theta})+
\frac{e^{s_j(\bm{x},\bm{\theta})}
}{q_{\bm{x}}(j)}\nabla_{\bm{\theta}}s_j(\bm{x},\bm{\theta})
}
{
\sum_{k'\in\mathcal{C}_{\bm{x}}}e^{s_{k'}(\bm{x},\bm{\theta})}+
\frac{e^{s_j(\bm{x},\bm{\theta})}}{q_{\bm{x}}(j)}
}\right];
\label{eq:practice1}
\end{align}
\ELSE 
\STATE $\bm{\theta}\leftarrow\bm{\theta}+\eta\nabla_{\bm{\theta}}\hat{R}$ with 
$\nabla_{\bm{\theta}}\hat{R}$ given by
{\small
\begin{align}
\nabla_{\bm{\theta}}s_y(\bm{x},\bm{\theta})-
\frac{
\sum_{k'\in\mathcal{C}_{\bm{x}}}
e^{s_{k'}(\bm{x},\bm{\theta})}\nabla_{\bm{\theta}}s_{k'}(\bm{x},\bm{\theta})+
\frac{e^{s_y(\bm{x},\bm{\theta})}}{q_{\bm{x}}(y)}\nabla_{\bm{\theta}}s_y(\bm{x},\bm{\theta})
}
{
\sum_{k'\in\mathcal{C}_{\bm{x}}}e^{s_{k'}(\bm{x},\bm{\theta})}+
\frac{e^{s_y(\bm{x},\bm{\theta})}}{q_{\bm{x}}(y)}
};
\label{eq:practice2}
\end{align}
}
\ENDIF
\ENDFOR
\end{algorithmic}
\end{algorithm}

Eq. (\ref{eq:obj}) suggests an efficient algorithm using a `doubly' stochastic gradient descend (SGD) method by sampling both the data points and classes. That is, by sampling a data point $(\bm{x},y)$, we find the candidate set $\mathcal{C}_{\bm{x}}\subset\{1,\cdots,K\}$. If $y\in\mathcal{C}_{\bm{x}}$, we sample $N_n$ noises from $\mathcal{N}_{\bm{x}}$ according to $q_{\bm{x}}$ and denote the selected noises as a set $T_{\bm{x}}$ ($|T_{\bm{x}}|=N_n$). We then optimize
\begin{align*}
\frac{1}{|T_{\bm{x}}|}
\sum_{j\in T_{\bm{x}}}
\log\frac{e^{s_y(\bm{x},\bm{\theta})}}
{
\sum_{k'\in\mathcal{C}_{\bm{x}}}e^{s_{k'}(\bm{x},\bm{\theta})}+e^{s_j(\bm{x},\bm{\theta})}/q_{\bm{x}}(j)
},
\end{align*}
with gradient $\nabla_{\bm{\theta}}\hat{R}$ given by Eq. \eqref{eq:practice1}.
Otherwise, if $y\not\in\mathcal{C}_{\bm{x}}$, we optimize
\begin{align*}
\log\frac{e^{s_y(\bm{x},\bm{\theta})}}
{
\sum_{k'\in\mathcal{C}_{\bm{x}}}e^{s_{k'}(\bm{x},\bm{\theta})}+e^{s_y(\bm{x},\bm{\theta})}/q_{\bm{x}}(y)
},
\end{align*}
with gradient $\nabla_{\bm{\theta}}\hat{R}$ given by Eq. \eqref{eq:practice2}.
This general procedure is provided in Algorithm \ref{algo:generic}. Algorithm \ref{algo:generic} has a complexity of $\mathcal{O}(N_c+N_n)$ (where $N_c=|\mathcal{C}_{\bm{x}}|$), which is independent of the class size $K$. In step 6, any method can be used to select $\mathcal{C}_{\bm{x}}$.

\subsection{Beam Tree Algorithm}

In the second algorithm, we provide an efficient way to find a competitive $\mathcal{C}_{\bm{x}}$. An attractive strategy is to use a tree defined on the classes, because one can perform fast heuristic search algorithms based on a tree structure to prune the uncompetitive classes. Indeed, any structure, e.g., graph or groups, can be used alternatively as long as the structure allows to efficiently prune uncompetitive classes. We will use tree structure for candidate selection in this paper.

Given a tree structure defined on the $K$ classes, the model $s_k(\bm{x},\bm{\theta})$ is interpreted as a tree model illustrated in Fig.~\ref{fig:model}. For simplicity, Fig.~\ref{fig:model} uses a binary tree over $K=8$ labels as example while any tree structure can be used for selecting $\mathcal{C}_{\bm{x}}$. In the example, circles denote internal nodes and squares indicate classes. The parameters are kept in the edges and denoted as $\bm{\theta}_{(o,c)}$, where $o$ indicates an internal node and $c$ is the index of the $c$-th child of node $o$. Therefore, a pair $(o,c)$ represents an edge from node $o$ to its $c$-th child. The dashed circles indicate that we do not keep any parameters in the internal nodes. 
\begin{figure}
\centering
\subfigure{
\includegraphics[scale=0.8]{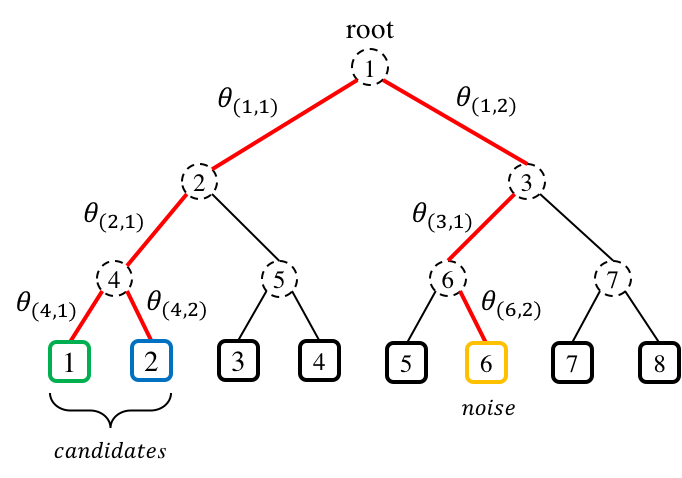}
}\vskip -0.1in
\caption{Illustration of the tree model. Suppose an example $(\bm{x},2)$ is arriving, and two candidate classes 1 and 2 are selected by beam search. The class 6 is sampled as noise.}
\vskip -0.1in
\label{fig:model}
\end{figure}
Now, define $s_k(\bm{x},\bm{\theta})$ as
\begin{align}
	s_k(\bm{x},\bm{\theta})=g_{\bm{\psi}}(\bm{x})\cdot\sum_{(o,c)\in\mathcal{P}_k}\bm{\theta}_{(o,c)},
	\label{eq:h}
\end{align}
where $g_{\bm{\psi}}(\bm{x})$ is a function parameterized by $\bm{\psi}$ and it maps the input $\bm{x}\sim\mathcal{X}$ to a representation $g_{\bm{\psi}}(\bm{x})\in\mathbb{R}^{d_r}$ for some $d_r$. For example, in image classification, a good choice of the representation $g_{\bm{\psi}}(\bm{x})$ of the raw pixels $\bm{x}$ is usually a deep neural network. $\mathcal{P}_k$ denotes the path from the root to the class $k$. Eq. (\ref{eq:h}) implies that the score of an example belonging to a class is calculated by summing up the scores along the corresponding path. Now, in Fig. \ref{fig:model}, suppose that we are given an example $(\bm{x},y)$ with class $y=2$ (blue). Using beam search, we find two candidates with high scores, i.e., class 1 (green) and class 2. Then, we let $\mathcal{C}_{\bm{x}}=\{1, 2\}$. In this case, we have $y\in\mathcal{C}_{\bm{x}}$, so we need to sample noises. Suppose we sample one class 6 (orange). According to Eq. (\ref{eq:practice1}), the parameters along the corresponding paths (red) will be updated.

\begin{algorithm}[t]
\small
\caption{The Beam Tree Algorithm.}
\label{algo:tree}
\begin{algorithmic}[1]
\STATE \textbf{Input:} $K$, $(\bm{x}_i, y_i)|_{i=1}^n$, representation function $g_{\bm{\psi}}(\bm{x})$, number of candidates $N_c=|\mathcal{C}_{\bm{x}}|$, number of sampled noises $N_n=|T_{\bm{x}}|$, sampling strategy $\bm{q}$ and learning rate $\eta$.
\STATE \textbf{Output:} $\hat{\bm{\theta}}$.
\vskip 0.1in
\STATE Construct a tree on the $K$ classes;
\STATE Initialize $\bm{\theta}$;
\FOR {every sampled example}
\STATE Receive example $(\bm{x},y)$;
\STATE Given $\bm{x}$, use beam search to find the $N_c$ classes with high scores to compose $\mathcal{C}_{\bm{x}}$;
\IF {$y\in\mathcal{C}_{\bm{x}}$}
\STATE Sample $N_n$ noises outside $\mathcal{C}_{\bm{x}}$ according to $\bm{q}$ and denote the selected noise set as $T_{\bm{x}}$;
\STATE Find the paths with respect to the classes in $\mathcal{C}_{\bm{x}}\cup T_{\bm{x}}$;
\ELSE 
\STATE Find the paths with respect to the classes in $\mathcal{C}_{\bm{x}}\cup\{y\}$; 
\ENDIF
\STATE Sum up the scores along each selected path for the corresponding class;
\STATE $\bm{\theta}_{(o,c)} \leftarrow \bm{\theta}_{(o,c)} + \eta\frac{\partial \hat{R}}{\partial \bm{\theta}_{(o,c)}}$ for each $(o, c)$ included in the selected paths according to Eqs. (\ref{eq:grad1}) and (\ref{eq:grad2});
\STATE $\bm{\psi}\leftarrow \bm{\psi} + \eta\frac{\partial \hat{R}}{\partial g}\frac{\partial g}{\partial\bm{\psi}}$; // \emph{if $g$ is parameterized.}
\ENDFOR
\end{algorithmic}
\end{algorithm}

Formally, given example $(\bm{x},y)$, if $y\in\mathcal{C}_{\bm{x}}$, we sample noises as a set $T_{\bm{x}}$. Then for $(o,c)\in\mathcal{P}_{\mathcal{C}_{\bm{x}}\cup T_{\bm{x}}}$, where $\mathcal{P}_{\mathcal{C}_{\bm{x}}\cup T_{\bm{x}}}=\bm{\cup}_{k\in\mathcal{C}_{\bm{x}}\cup T_{\bm{x}}}\mathcal{P}_k$, the gradient with respect to $\bm{\theta}_{(o,c)}$ is
\begin{align}
\frac{\partial \hat{R}}{\partial \bm{\theta}_{(o,c)}}
=\frac{1}{|T_{\bm{x}}|}
\sum_{j\in T_{\bm{x}}}
\Bigg[
\mathbb{I}\left((o,c)\in\mathcal{P}_{y} \right)-
\frac{
\sum_{k'\in\mathcal{C}_{\bm{x}}} \mathbb{I}((o,c)\in \mathcal{P}_{k'}) e^{s_{k'}(\bm{x},\bm{\theta})}+
\mathbb{I}((o,c)\in \mathcal{P}_j)\frac{e^{s_j(\bm{x},\bm{\theta})}}{q_{\bm{x}}(j)}
}
{
\sum_{k'\in\mathcal{C}_{\bm{x}}}e^{s_{k'}(\bm{x},\bm{\theta})}+\frac{e^{s_j(\bm{x},\bm{\theta})}}{q_{\bm{x}}(j)}
}
\Bigg]
g_{\bm{\psi}}(\bm{x}).
\label{eq:grad1}
\end{align}
Note that an edge may be included in multiple selected paths. For example, $\mathcal{P}_1$ and $\mathcal{P}_2$ share edges $(1,1)$ and $(2,1)$ in Fig. \ref{fig:model}. The case of $y\not\in \mathcal{C}_{\bm{x}}$ can be illustrated similarly. The gradient with respect to $\bm{\theta}_{(o,c)}$ when $y\not\in\mathcal{C}_{\bm{x}}$ is
\begin{align}
\frac{\partial \hat{R}}{\partial \bm{\theta}_{(o,c)}}=
\Bigg[
\mathbb{I}\left((o,c)\in\mathcal{P}_{y} \right)- 
\frac{
\sum_{k'\in\mathcal{C}_{\bm{x}}} \mathbb{I}((o,c)\in \mathcal{P}_{k'}) e^{s_{k'}(\bm{x},\bm{\theta})}+
\mathbb{I}((o,c)\in \mathcal{P}_y)\frac{e^{s_y(\bm{x},\bm{\theta})}}{q_{\bm{x}}(y)}
}
{
\sum_{k'\in\mathcal{C}_{\bm{x}}}e^{s_{k'}(\bm{x},\bm{\theta})}+\frac{e^{s_y(\bm{x},\bm{\theta})}}{q_{\bm{x}}(y)}
}
\Bigg]
g_{\bm{\psi}}(\bm{x}).
\label{eq:grad2}
\end{align}
The gradients in Eqs. (\ref{eq:grad1}) and (\ref{eq:grad2}) enjoy the following property.
\begin{prop}\label{pro:1}
At each iteration of Algorithm \ref{algo:tree}, if an edge $(o,c)$ is included in every selected path, then $\bm{\theta}_{(o,c)}$ does not need to be updated.
\end{prop}
The proof of Proposition \ref{pro:1} is straightforward that if $(o,c)$ belongs to every selected path, then the gradients in Eqs. (\ref{eq:grad1}) and (\ref{eq:grad2}) are $\bm{0}$. The above property allows a fast detection of those parameters which do not need to be updated in SGD and hence can save computations. In practice, the number of shared edges is related to the tree structure.

Since we use beam search to choose the candidates in a tree structure, the proposed algorithm is referred to as {\em Beam Tree}, which is depicted in Algorithm \ref{algo:tree}.
\footnote{\small The beam search procedure in step 7 is provided in the supplementary material.}
For the tree construction method in step 3, we can use some hierarchical clustering based methods which will be detailed in the experiments and supplementary material. In the algorithm, the beam search needs $\mathcal{O}\left(N_c\log_{b}K\right)$ operations, where $b$ is a constant related to the tree structure, e.g., binary tree for $b=2$. The parameter updating needs $\mathcal{O}((N_c+N_n)\log_{b}K)$ operations. Therefore, Algorithm \ref{algo:tree} has a complexity of $\mathcal{O}((2N_c+N_n)\log_{b}K)$ which is logarithmic with respect to $K$. The term  $\log_{b}K$ is from the tree structure used in this specific candidate selection method, so it does not conflict with the complexity of the general Algorithm \ref{algo:generic}, which is independent of $K$. Another advantage of the Beam Tree algorithm is that it allows fast predictions and can naturally output the top-$J$ predictions using beam search. The prediction time has an order of $\mathcal{O}\left(J\log_{b}K\right)$ for the top-$J$ predictions. 

\section{Application to Neural Language Modeling}\label{sec:nlp}
In this section, we apply the CANE method to neural language modeling which solves a probability density estimation problem. In neural language models, the conditional probability distribution of the target word $w$ given context $h$ is defined as
\begin{align*}
	P_h(w) = \frac{e^{s_w(h, \bm{\theta})}}{\sum_{w'\in\mathcal{V}}e^{s_{w'}(h,\bm{\theta})}},
\end{align*}
where $s_{w}(h,\bm{\theta})$ is the scoring function with parameter $\bm{\theta}$. A word $w$ in the context $h$ will be represented by an embedding vector $\bm{u}_w\in\mathbb{R}^d$ with embedding size $d$. Given context $h$, the model computes the score for the target word $w$ as
\begin{align*}
	s_w(h,\bm{\theta})=g_{\bm{\varphi}}(\bm{u}_h)\bm{v}_w,
\end{align*}
where $\bm{\theta}=\{\bm{u},\bm{v},\bm{\varphi}\}$, $g_{\bm{\varphi}}(\cdot)$ is a representation function (parameterized by $\bm{\varphi}$) of the embeddings in the context $h$, e.g., a LSTM modular \cite{hochreiter1997long}, and $\bm{v}_w$ is the weight parameter for the target word $w$. Both the word embedding $\bm{u}$ and weight parameter $\bm{v}$ need to be estimated. In language models, the vocabulary size $|\mathcal{V}|$ is usually very large and the computation of the normalization factor is expensive. Therefore, instead of estimating the exact probability distribution $P_h(w)$, sampling methods such as NCE and its variants \cite{mnih2013learning,ji2015blackout} are typically adopted to approximate $P_h(w)$.

In order to apply the CANE method, we need to select the candidates given any context $h$. For multi-class classification problem, we have devised a Beam Tree algorithm in Algorithm \ref{algo:tree} that uses a tree structure to select candidates, and the tree can be obtained by some hierarchical clustering methods over $\bm{x}$ before learning. However, different from the classification problem, the word embeddings in the language model are not known before training, and thus obtaining a hierarchical structure based on the word embeddings is not practical.
In this paper, we construct a simple tree with only one layer under the root, where the layer contains $N$ subsets formed by splitting the words according to their frequencies. At each iteration of Algorithm~\ref{algo:tree}, we route the example by selecting the subset with the largest score (in place of beam search) and then sample the candidates from the subset according to some distribution. For the noises in CANE, we directly sample words out of the candidate set according to $\bm{q}$. 
Other methods can be used to select the candidates alternatively, for example, one can choose candidates conditioned on the context $h$ using a lightly pre-trained N-gram model.

\section{Related Algorithms}
We provide a discussion comparing CANE with the existing techniques for solving the large class space problem. Given $(\bm{x},y)$, NCE and its variants \cite{gutmann2010noise,mnih2013learning,mikolov2013distributed,ji2015blackout,Titsias2016one,botev2017complementary} use the observed class $y$ as the only `candidate', while CANE chooses a subset of candidates $\mathcal{C}_{\bm{x}}$ according to $\bm{x}$. NCE assumes the entire noise distribution $P_{noise}(y)$ is known (e.g., a power-raised unigram distribution).
However, in general multi-class classification problems, when the knowledge of the noise distribution is absent, NCE may have unstable estimations using an inaccurate noise distribution. CANE is developed for general multi-class classification problems and does not rely on a known noise distribution. Instead, CANE focuses on a small candidate set $\mathcal{C}_{\bm{x}}$. Once the true class label is contained in $\mathcal{C}_{\bm{x}}$ with high probability, CANE will have low statistical variance. 
The variants of NCE \cite{mikolov2013distributed,ji2015blackout,Titsias2016one,botev2017complementary} also sample one or multiple noises to replace the normalization factor while according theoretical guarantees on the consistency and variance are rarely discussed. NCE and its variants can not speed up prediction while the Beam Tree algorithm can reduce the prediction complexity to $O(\log K)$.

The Beam Tree algorithm is related to some tree classifiers, while CANE is a general procedure and we only use tree structure to select candidates. The Beam Tree method itself is also different from existing tree classifiers. Most of the state-of-the-art tree classifiers, e.g., LOMTree \cite{choromanska2015logarithmic} and Recall Tree \cite{daume2016logarithmic}, store local classifiers in their internal nodes, and route examples through the root until reaching the leaf. Differently, the Beam Tree algorithm shown in Fig. \ref{fig:model} does not maintain local classifiers, and it only uses the tree structure to perform global heuristic search for candidate selection. We will compare our approach to some state-of-the-art tree classifiers in the experiments.

\section{Experiments}
We evaluate the CANE method in various applications in this section, including both multi-class classification problems and neural language modeling. We compare CANE with NCE, its variants and some state-of-the-art tree classifiers that have been used for large class space problems. The competitors include the standard softmax, the NCE \cite{mnih2013learning,mnih2012fast}, the BlackOut \cite{ji2015blackout}, the hierarchical softmax (HSM) \cite{morin2005hierarchical}, the Filter Tree \cite{beygelzimer2009error} implemented in Vowpal-Wabbit (VW, a learning platform)\footnote{\url{https://github.com/JohnLangford/vowpal_wabbit/wiki}}, the LOMTree \cite{choromanska2015logarithmic} in VW and the Recall Tree \cite{daume2016logarithmic} in VW.

\subsection{Classification Problems}

\begin{figure*}[!ht]
\centering
\subfigure[Sector]{
\includegraphics[scale=0.4]{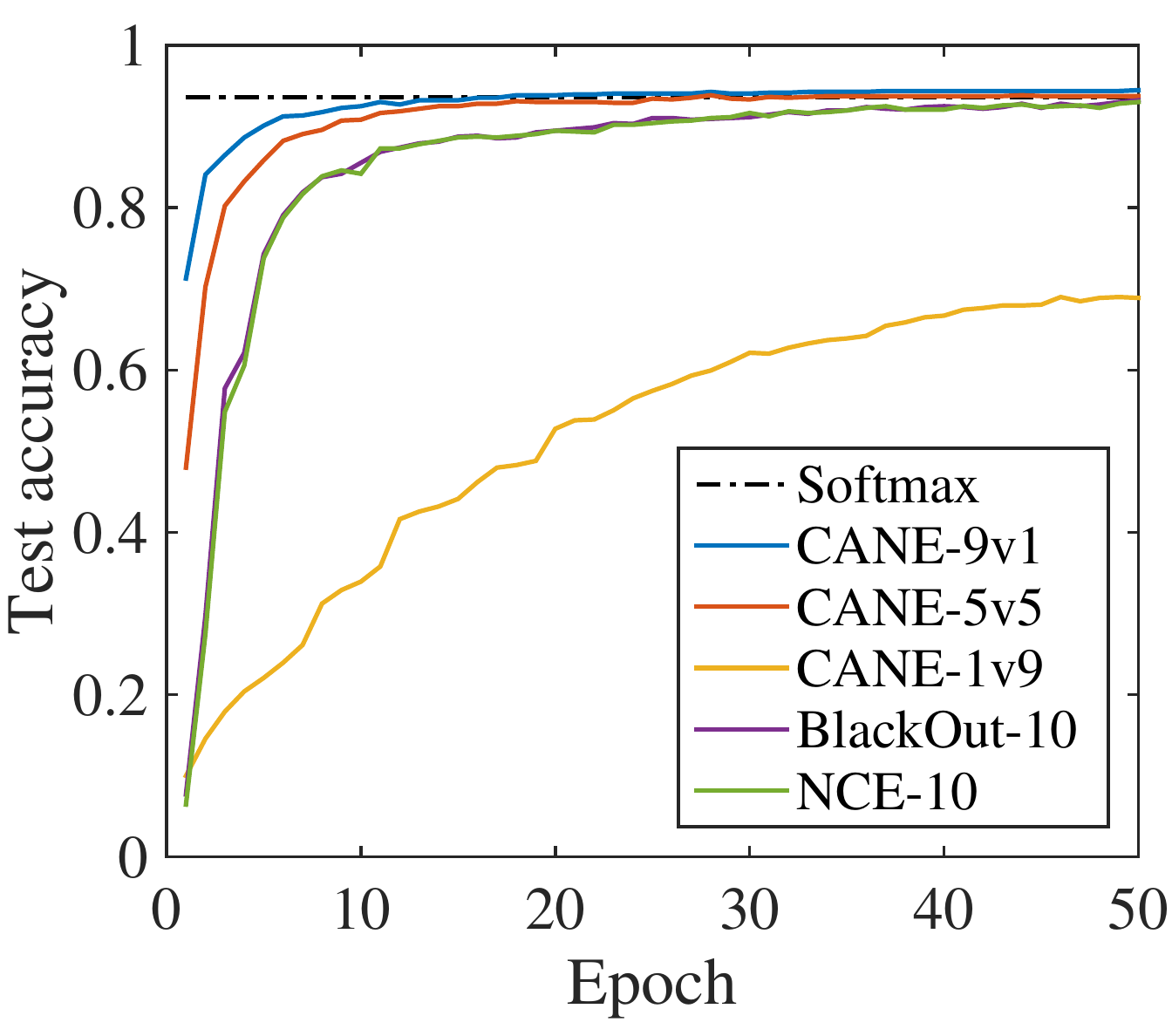}
}
\subfigure[ALOI]{
\includegraphics[scale=0.4]{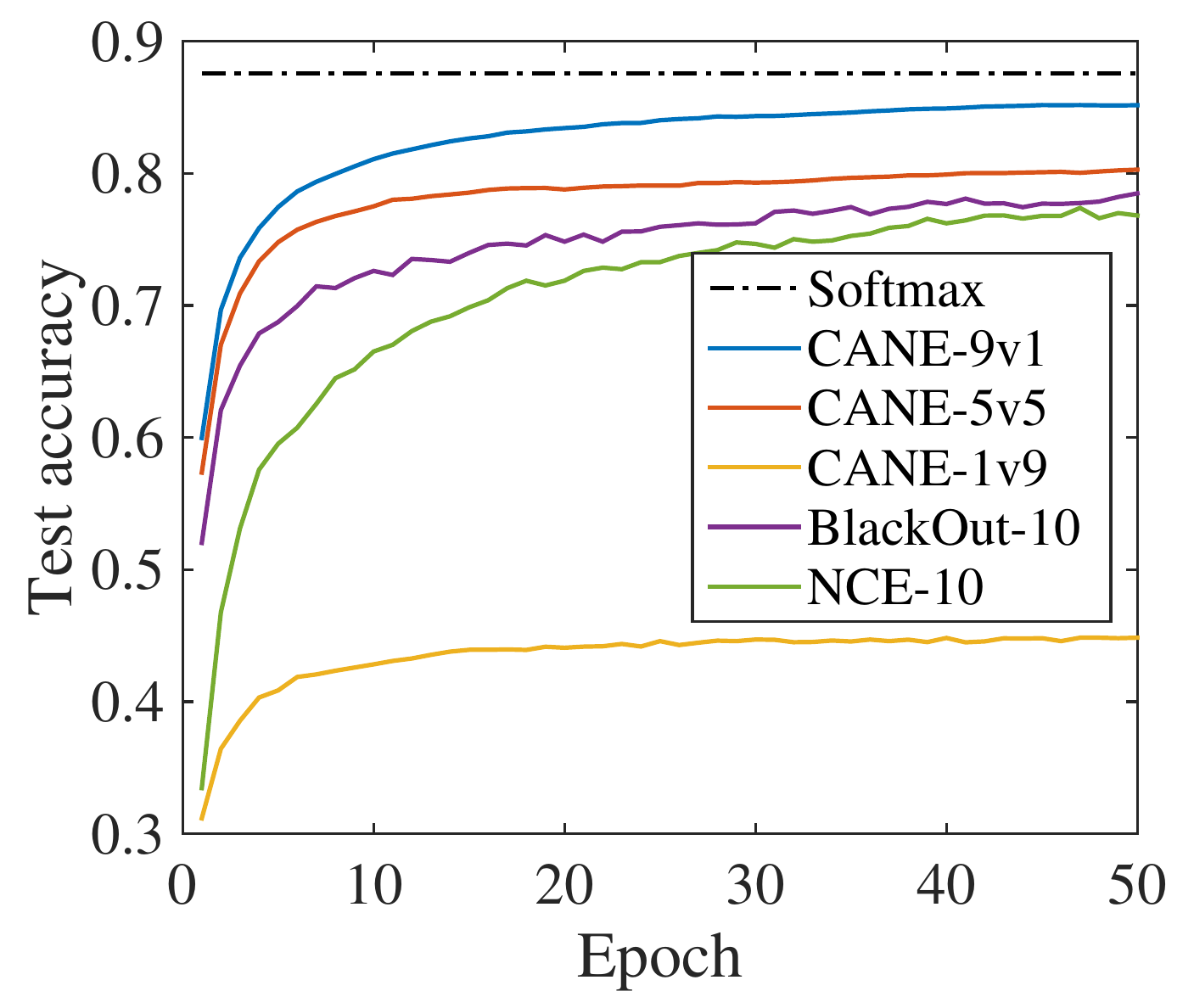}
}

\subfigure[ImgNet-2010]{
\includegraphics[scale=0.4]{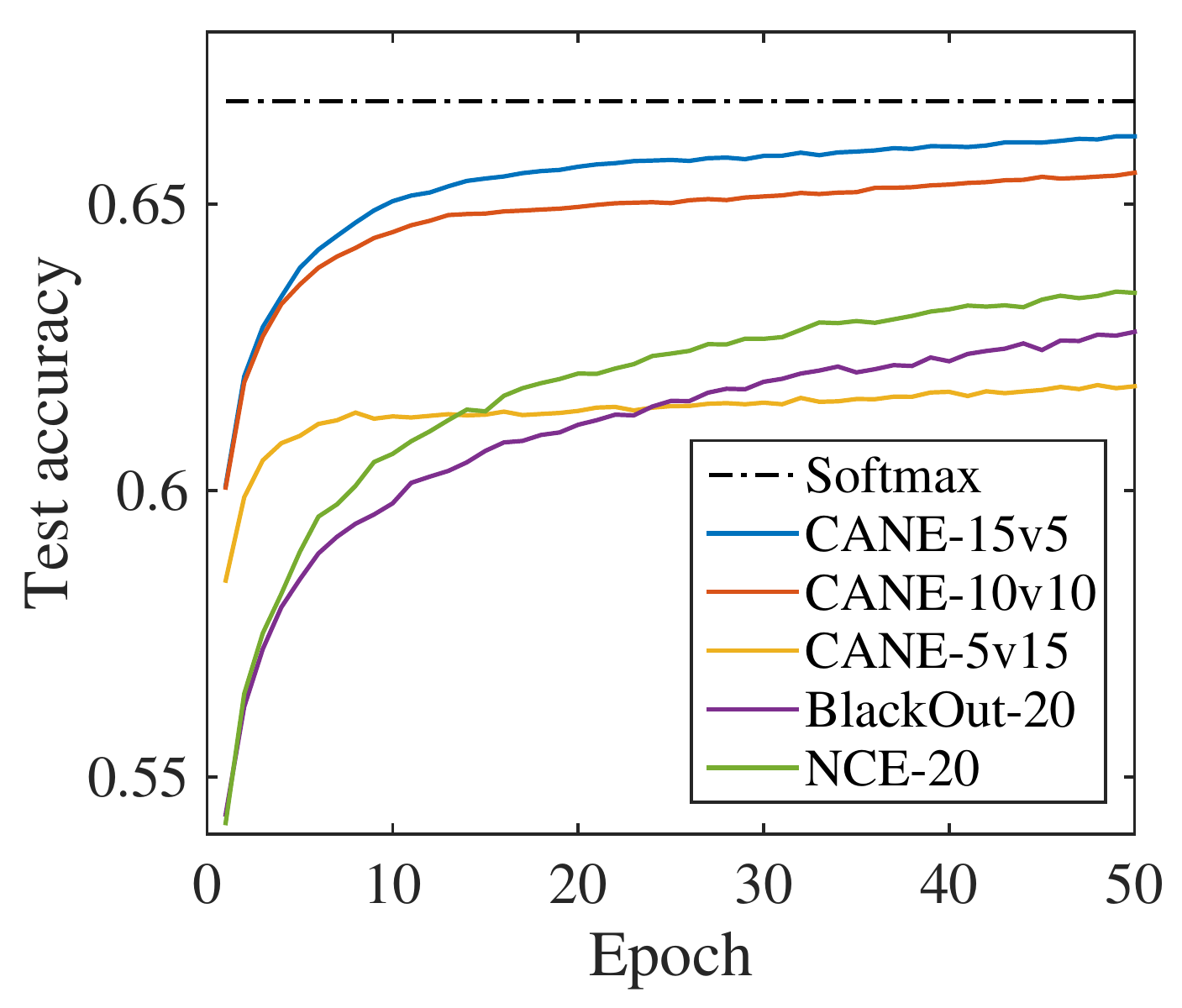}
}
\subfigure[ImgNet-10K]{
\includegraphics[scale=0.39]{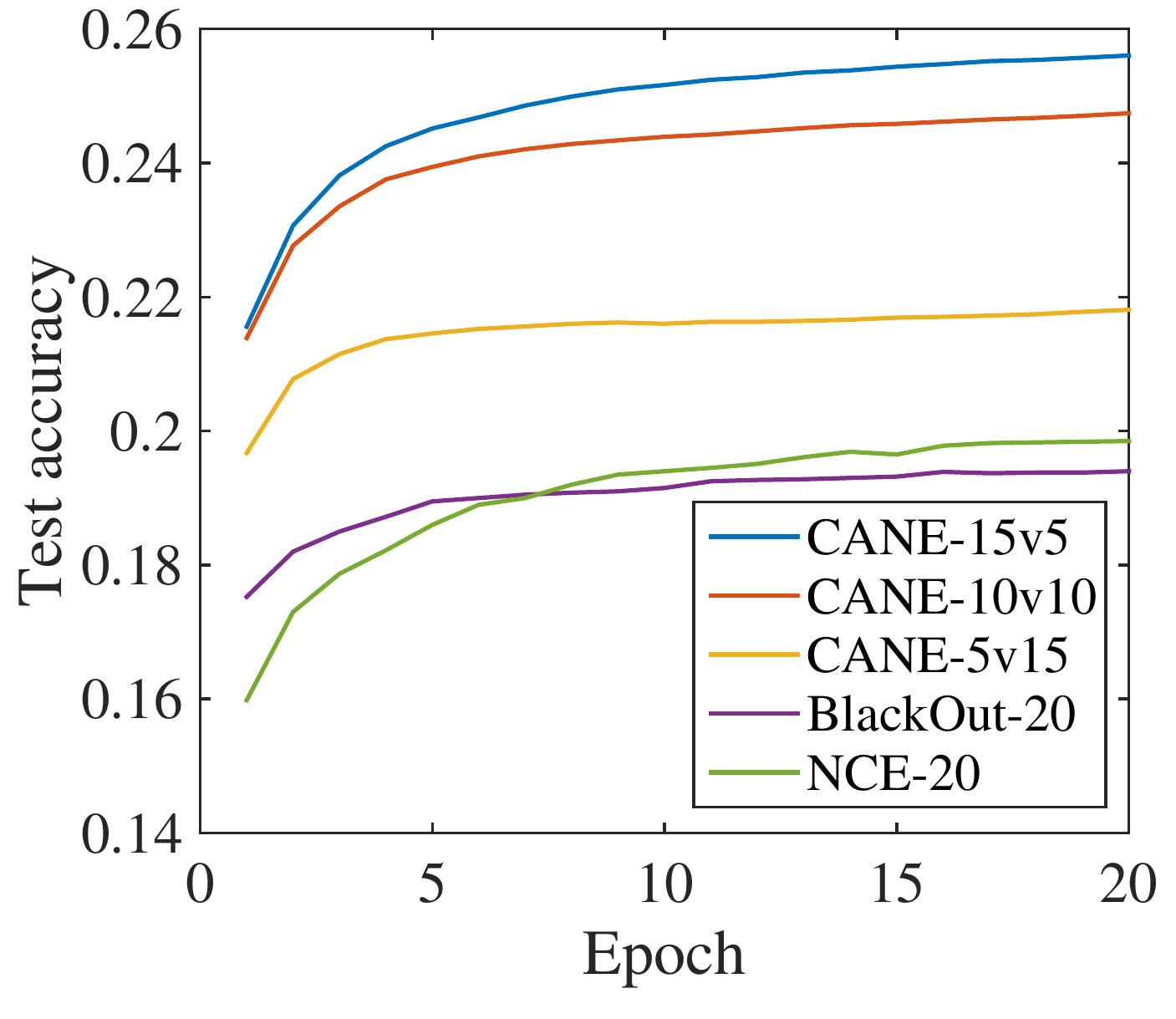}
}

\caption{Results of test accuracy vs. epoch on different classification datasets.}
\vskip -0.2in
\label{fig:class}
\end{figure*}

\begin{table*}[!ht]
\small
\centering
\caption{Training / testing time of the sampling methods. Running Softmax and testing NCE and BlackOut on large datasets are time consuming. We use multi-thread implementation for these methods and estimate the running time. `$\sim$' indicates the estimated time.}
\scalebox{1}{
\begin{tabular}{c|c|c|c|c|c||c}
\hline
\hline
Data & NCE-10 & BlackOut-10 & CANE-1v9 & CANE-5v5 & CANE-9v1 & Softmax \\ 
\hline
Sector & 0.4m / 0.8s & 0.4m / 0.8s & 1m / 0.1s & 1.8m / 0.1s & 2.3m / 0.2s & 6.1m / 0.9s \\
ALOI & 3m / 6s & 3m / 6s & 4m / 0.1s & 7m / 0.3s & 8m / 0.5s & 28m / 7s \\
\hline
\hline
Data & NCE-20 & BlackOut-20 & CANE-5v15 & CANE-10v10 & CANE-15v5 & Softmax \\ 
\hline
ImgNet-2010 & 3.5h / 8m & 3.5h / 8m & 4h / 0.4m & 5.8h / 0.7m & 6.4h / 0.9m & 96h / 8.7m \\
ImgNet-10K & 13h / $\sim$5d & 12h / $\sim$5d & 20h / 1h & 33h / 1.5h & 39h / 2h & $\sim$140d / $\sim$5d \\
\hline
\hline
\end{tabular}
}
\label{tb:1}
\end{table*}

\begin{table}[!ht]
\small
\centering
\caption{Accuracy and training / testing time of the tree classifiers.}
\begin{tabular}{c|c|c|c|c}
\hline
\hline
Data & HSM & Filter Tree & LOMTree & Recall Tree \\ 
\hline
Sector & 91.36\% & 84.67\% & 84.91\% & 86.89\% \\
 & 0.5m / 0.1s & 0.4m / 0.4s & 0.5m / 0.2s &  0.7m / 0.2s \\
\hline
ALOI & 65.69\% & 20.07\% & 82.70\% &  83.03\% \\
 & 1m / 0.4s & 1m / 0.2s & 3.3m / 1s & 2.5m / 0.2s \\
\hline
ImgNet & 47.68\% & 48.29\% & 49.87\% & 61.28\% \\
2010 & 4.7h / 0.5m & 6.8h / 0.1m & 17.8h / 0.3m & 32h / 0.5m \\
\hline
ImgNet & 17.31\% & 4.49\% & 9.72\% & 22.74\% \\
10K & 14h / 1h & 22h / 0.3h & 23h / 0.3h & 68h / 1.2h \\
\hline
\hline
\end{tabular}
\label{tb:2}
\end{table}

In this section, we consider four multi-class classification problems, including the Sector\footnote{\url{http://www.cs.utexas.edu/~xrhuang/PDSparse/}} dataset with 105 classes \cite{chang2011libsvm}, the ALOI\footnote{\url{http://www.csie.ntu.edu.tw/~cjlin/ libsvmtools/datasets/multiclass.html}} dataset with 1000 classes \cite{geusebroek2005amsterdam}, the ImageNet-2010\footnote{\url{http://image-net.org}} dataset with 1000 classes, and the ImageNet-10K\footnotemark[5] dataset with 10K classes (ImageNet Fall 2009 release). The data from Sector and ALOI is split into 90\% training and 10\% testing. In ImageNet-2010, the training set contains 1.3M images and we use the validation set containing 50K images as the test set. The ImageNet-10K data contains 9M images and we randomly split the data into two halves for training and testing by following the protocols in \cite{deng2010does,sanchez2011high,le2013building}.
For ImageNet-2010 and ImageNet-10K datasets, similar to \cite{oquab2014learning}, we transfer the mid-level representations from the pre-trained VGG-16 net \cite{simonyan2014very} on ImageNet 2012 data \cite{russakovsky2015imagenet} to our case. Then, we concatenate CANE or other compared methods above the partial VGG-16 net as the top layer. The parameters of the partial VGG-16 net are pre-trained\footnote{\small \url{http://www.robots.ox.ac.uk/~vgg/research/very_deep/}} and kept fixed. Only the parameters in the top layer are trained on the target datasets, i.e., ImageNet-2010 and ImageNet-10K. 

We use $b$-nary tree for CANE and set $b=10$ for all classification problems. We trade off $|\mathcal{C}_{\bm{x}}|$ and $|T_{\bm{x}}|$ to see how these parameters affect the learning performance. Different configurations will be referred to as `CANE-($|\mathcal{C}_{\bm{x}}|$ vs. $|T_{\bm{x}}|$)'. We always let $|\mathcal{C}_{\bm{x}}|+|T_{\bm{x}}|$ equal the number of noises used by NCE and BlackOut, so that these methods will have the same number of considered classes. We use `NCE-$k$' and `BlackOut-$k$' to denote the corresponding method with $k$ noises. Generally, a large $|\mathcal{C}_{\bm{x}}|+|T_{\bm{x}}|$ and $k$ will lower the variance of CANE, NCE and BlackOut and improve their performance, but this also increases the computation. 
We set $k=10$ for Sector and ALOI and $k=20$ for ImageNet-2010 and ImageNet-10K.
We uniformly sample noises in CANE. For NCE and BlackOut, by following \cite{mnih2012fast,mnih2013learning,ji2015blackout,botev2017complementary}, we use the power-raised unigram distribution with the power factor selected from $\{0, 0.1, 0.3, 0.5, 0.75, 1\}$ to sample the noises. However, when the classes are balanced as in many cases of the classification datasets, this distribution reduces to the uniform distribution. For the compared tree classifiers, the HSM adopts the same tree used by CANE, the Filter Tree generates a fixed tree itself in VW, the LOMTree and Recall Tree use binary trees and they are able to adjust the tree structure automatically.

All the methods use SGD with learning rate selected from $\{0.0001, 0.001, 0.01, 0.05, 0.1, 0.5,$ $ 1.0\}$. The Beam Tree algorithm requires a tree structure and we use some tree generated by a simple hierarchical clustering method on the centers of the individual classes.\footnote{\small The method is provided in the supplementary material.}
We run all the methods 50 epochs on Sector, ALOI and ImageNet-2010 datasets and 20 epochs on ImageNet-10K to report the accuracy vs. epoch curves. All the methods are implemented using a standard CPU machine with quad-core Intel Core i5 processor.

Fig. \ref{fig:class} and Table \ref{tb:1} show the accuracy vs. epoch plots and the training / testing time for NCE, BlackOut, CANE and Softmax. The tree classifiers in the VW platform require the number of training epochs as input and do not take evaluation directly after each epoch, so we report the final results of the tree classifiers in Table \ref{tb:2}. For ImageNet-10K data, the Softmax method is very time consuming (even with multi-thread implementation) and we do not report this result. As we can observe, by fixing $|\mathcal{C}_{\bm{x}}|+|T_{\bm{x}}|$, using more candidates than noises in CANE will achieve better performance, because a larger $\mathcal{C}_{\bm{x}}$ will increase the chance to cover the target class $y$. The probability that the target class is included in the selected candidate set on the test data is reported in Table \ref{tb:topcandi}. 
On all the datasets, CANE with larger candidate set achieves considerable improvement compared to other methods in terms of accuracy.
The speed of processing each example of CANE is slightly slower than that of NCE and BlackOut because of beam search, however, CANE shows faster convergence to reach higher accuracy. 
Moreover,  the prediction time of CANE is much faster than those of NCE and BlackOut. It is worth mentioning that CANE exceeds some state-of-the-art results on the ImageNet-10K data, e.g., 19.2\% top-1 accuracy reported in \cite{le2013building} and 21.9\% top-1 accuracy reported in \cite{mensink2013distance} which are conducted from $\mathcal{O}(K)$ methods; but it underperforms the recent result 28.4\% in \cite{huang2016local}. This is probably because the VGG-16 net works better than the neural network structure used in \cite{le2013building} and the distance-based method in \cite{mensink2013distance}, while the method in \cite{huang2016local} adopts a better feature embedding, which leads to superior prediction performance on this dataset.

\begin{table}[t]
\small
\centering
\caption{The probability that the true label is included in the selected candidate set on the test set, i.e., the top-$|\mathcal{C}_{\bm{x}}|$ accuracy.}
\begin{center}
\begin{tabular}{c|cc|c|ccccccc}
\hline
\hline
$|\mathcal{C}_{\bm{x}}|$ & Sector & ALOI & $|\mathcal{C}_{\bm{x}}|$ & ImgNet-2010 & ImgNet-10K \\ 
\hline
1 & 68.89\% & 44.84\% & 5 & 76.59\% & 39.59\% \\
5 & 96.57\% & 86.47\% & 10 & 87.29\% & 53.28\% \\
9 & 97.92\% & 93.59\% & 15 & 91.17\% & 60.22\% \\
\hline
\hline
\end{tabular}
\end{center}
\label{tb:topcandi}
\end{table}

\subsection{Neural Language Modeling}
In this experiment, we apply the CANE method to neural language modeling. We test the methods on two benchmark corpora: the Penn TreeBank (PTB) \cite{mikolov2010recurrent} and Gutenberg\footnote{\small \url{www.gutenberg.org}} corpora. The Penn TreeBank dataset contains 1M tokens and we choose the most frequent 12K words appearing at least 5 times as the vocabulary. The Gutenberg dataset contains 50M tokens and the most frequent 116K words appearing at least 10 times are chosen as the vocabulary.
We set the embedding size as 256 and use a LSTM model with 512 hidden states and 256 projection size. The sequence length is fixed as 20 and the learning rate is selected from $\{0.025, 0.05, 0.1, 0.2\}$.

The tree classifiers evaluated in multi-class classification problems can not be directly applied to solve the language modeling problem, so we omit their comparison and focus on the evaluation of the sampling methods. We sample 40, 60 and 80 noises for NCE and Blackout respectively and use power-raised unigram distribution with the power factor selected from $\{0, 0.25, 0.5, 0.75, 1\}$. For CANE, we adopt the one-layer tree structure discussed in Section~\ref{sec:nlp} with $N=6$ subsets, split by averaging over the word frequencies. We uniformly sample the candidates when reaching any subset. For efficiency consideration, we respectively sample 40, 60 and 80 candidates plus one more uniform noise for CANE. The experiments in this section are implemented on a machine with NVIDIA Tesla M40 GPUs.

\begin{figure}[t]
\centering
\subfigure[PTB (80)]{
\includegraphics[scale=0.35]{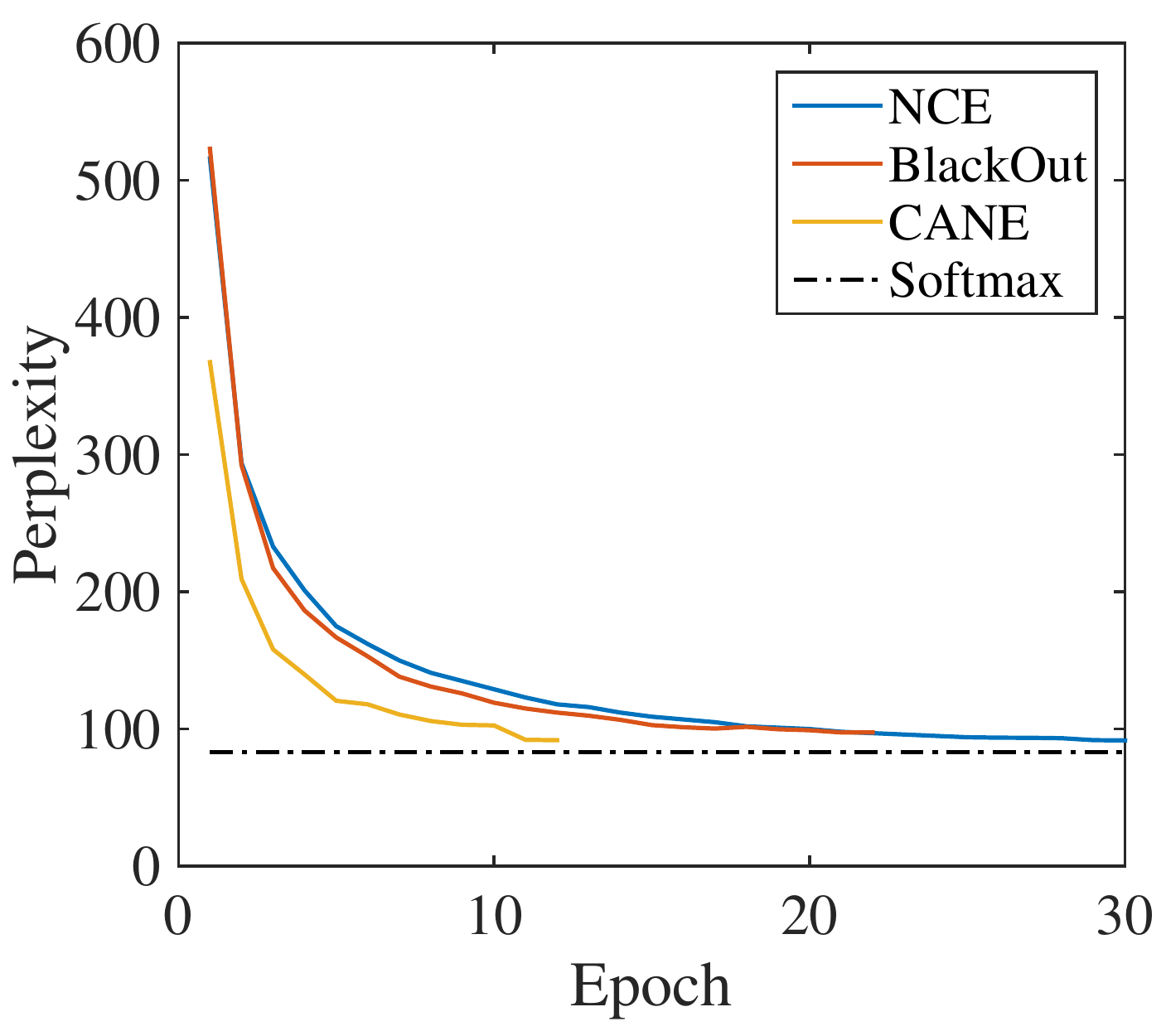}
}
\subfigure[PTB (60)]{
\includegraphics[scale=0.35]{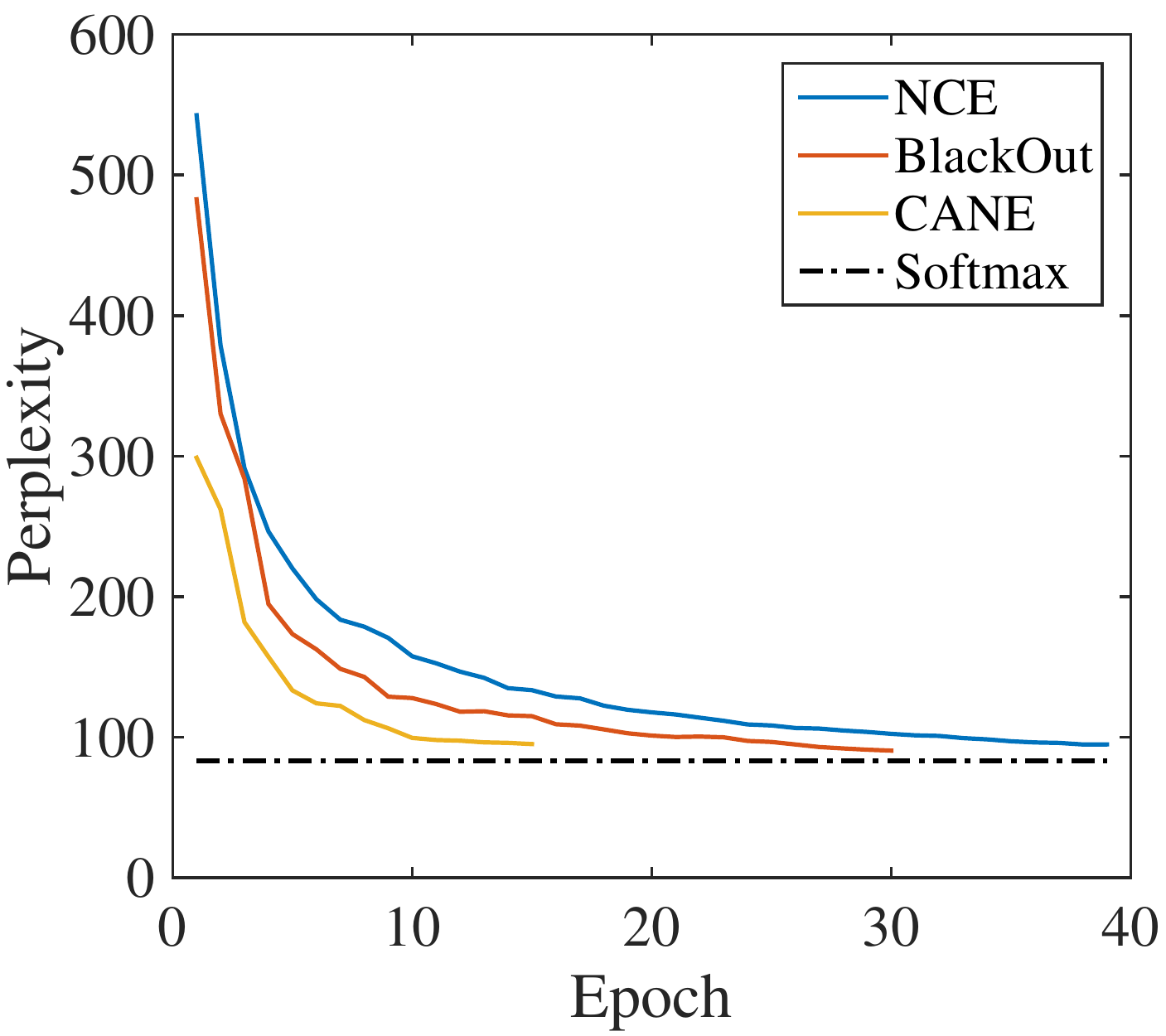}
}
\subfigure[PTB (40)]{
\includegraphics[scale=0.355]{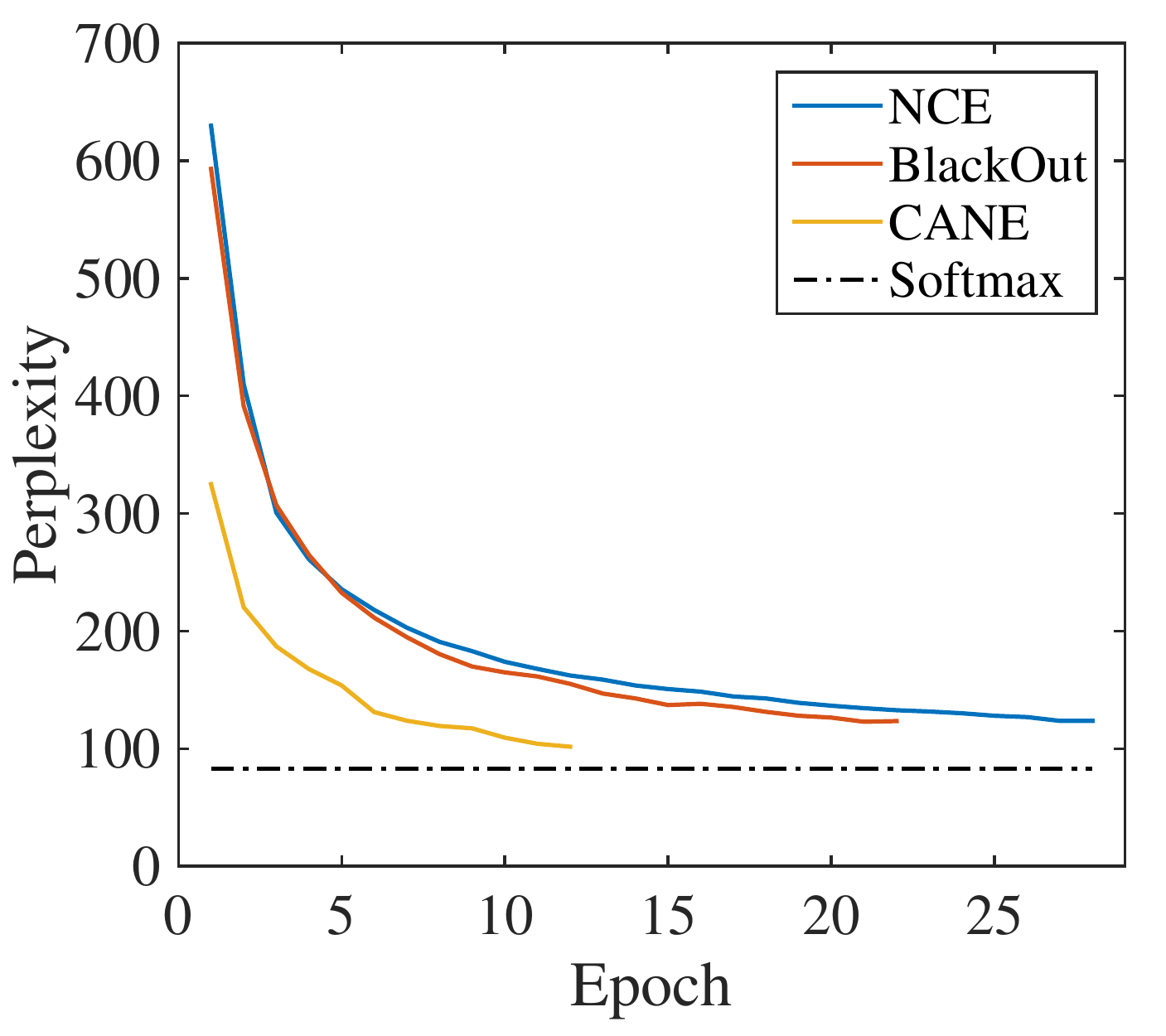}
}

\subfigure[Gutenberg (80)]{
\includegraphics[scale=0.35]{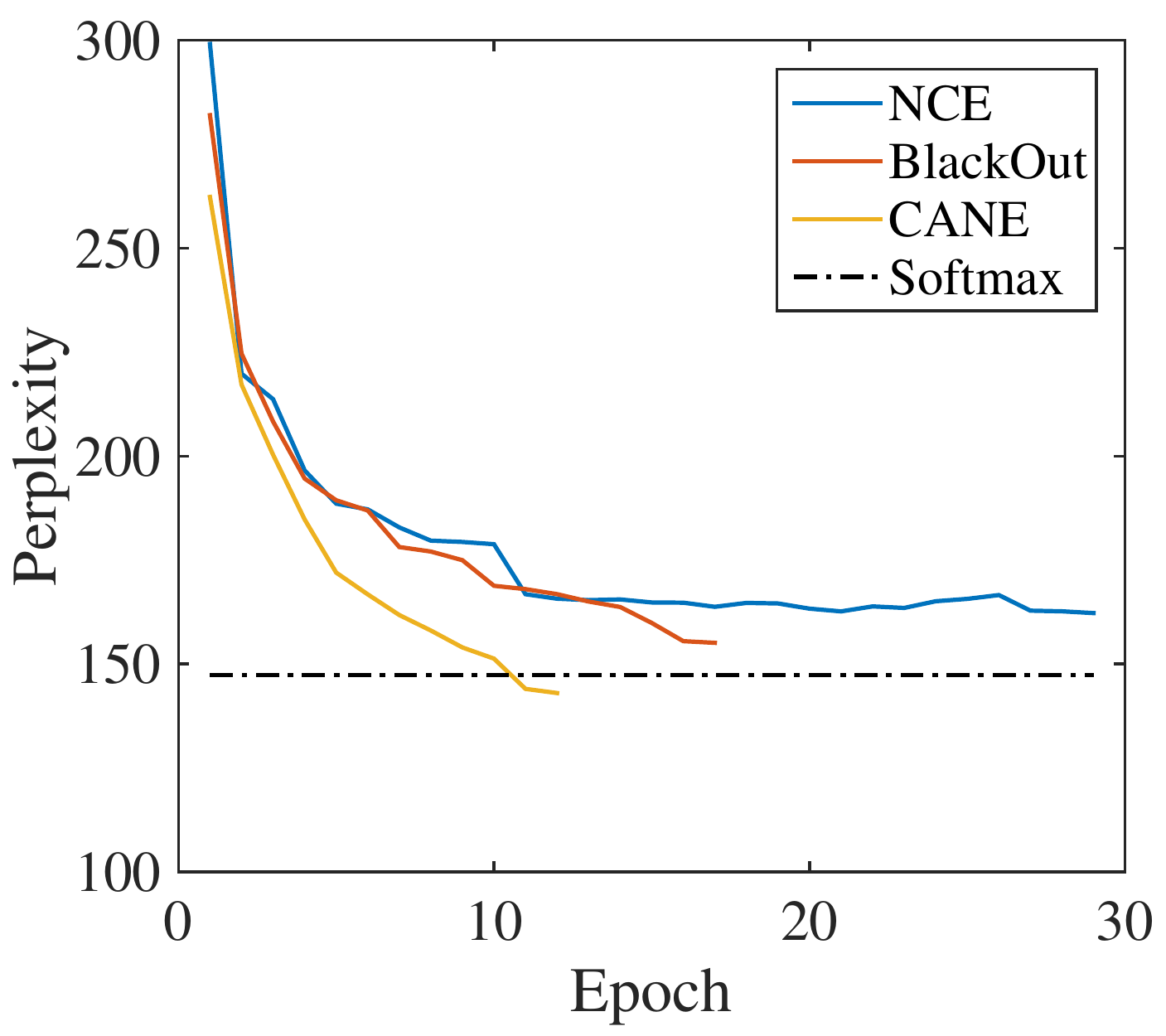}
}
\subfigure[Gutenberg (60)]{
\includegraphics[scale=0.355]{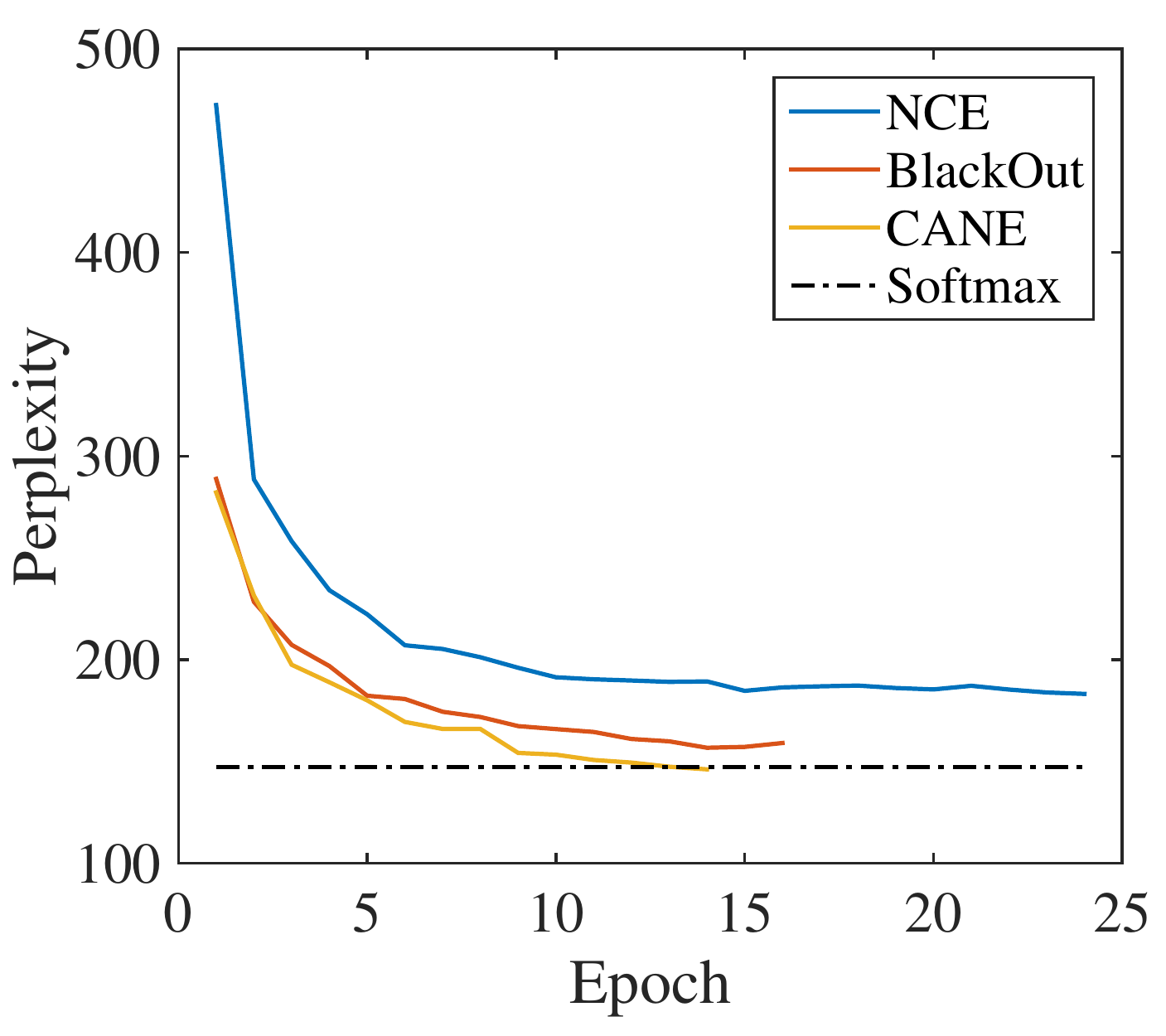}
}
\subfigure[Gutenberg (40)]{
\includegraphics[scale=0.38]{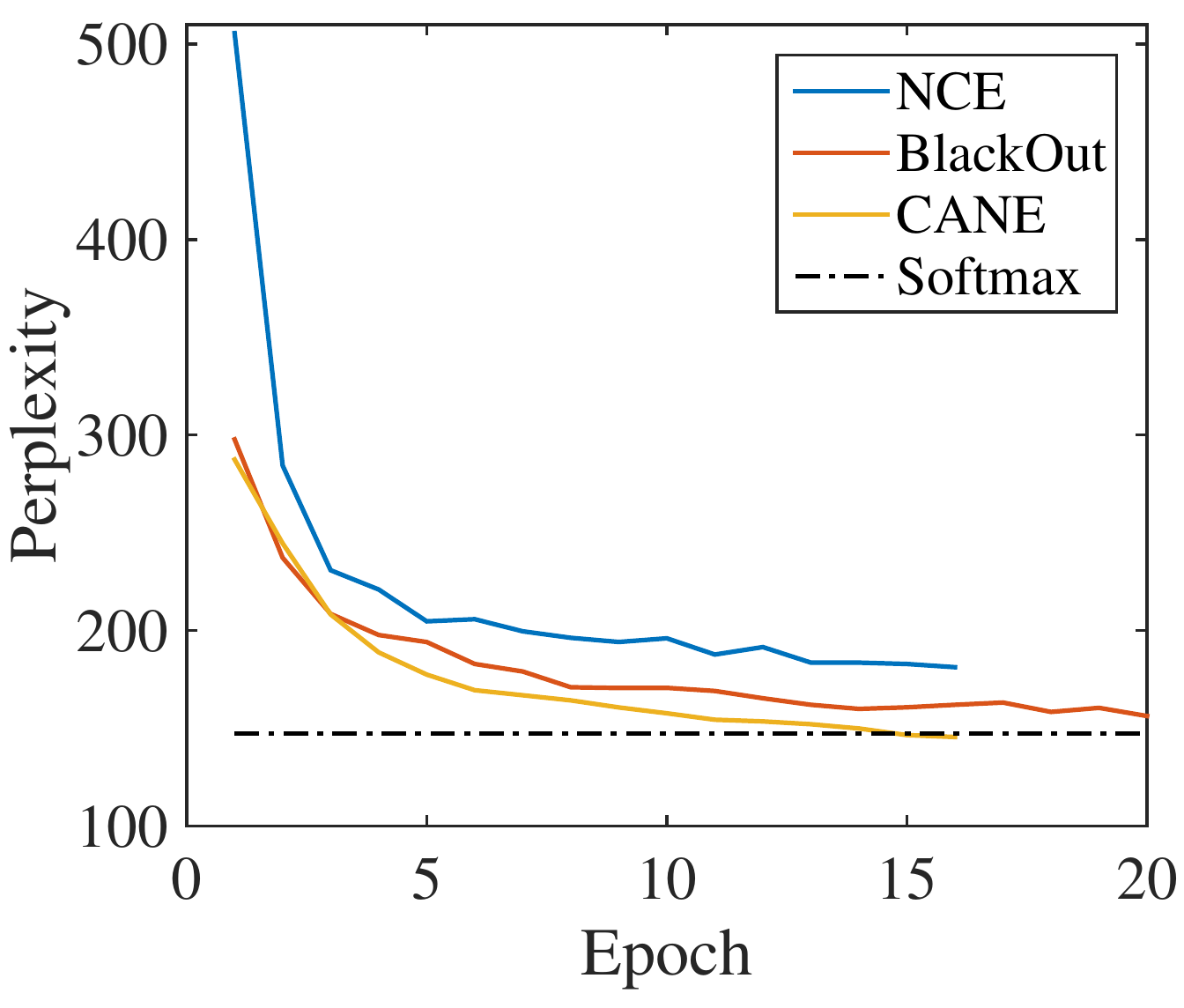}
}

\caption{Test perplexity vs. training epoch on PTB and Gutenberg datasets. Numbers in the brackets indicate the number of selected candidates / noises.}
\vskip -0.1in
\label{fig:nlp}
\end{figure}

The test perplexities are shown in Fig. \ref{fig:nlp}. As we can observe, the CANE method always achieves faster convergence and lower perplexities (approaching that of Softmax) compared to NCE and Blackout under various settings. Generally, when the number of selected candidates / noises decrease, the test perplexities of all the methods increase on both datasets, while the performance degradation of CANE is not obvious. 
By using GPUs, all the methods can finish training within a few minutes on the PTB dataset; for the Gutenberg corpus, CANE and BlackOut have similar training time that is around 5 hours on all the three settings, while NCE spends around 6-8 hours on these tasks and Softmax uses 35 hours to finish the training.

\section{Conclusion}
We proposed Candidates vs. Noises Estimation (CANE) for fast learning in multi-class classification problems with many labels and applied this method to the word probability estimation problem in neural language models. We showed that CANE is consistent and the computation using SGD is always efficient (that is, independent of the class size $K$).  
Moreover, the new estimator has low statistical variance approaching that of the softmax logistic regression, if the observed class label belongs to the candidate set with high probability. Empirical results demonstrated that CANE is effective for speeding up both training and prediction in multi-class classification problems and CANE is effective in neural language modeling.
We note that this work employs a fixed distribution (i.e., the uniform distribution) to sample noises in CANE. However it can be very useful in practice to estimate the noise distribution, i.e., $\bm{q}$, during training, and select noise classes according to this distribution. 

\clearpage

\bibliography{cane}
\bibliographystyle{unsrt}

\clearpage
\appendix
\section*{Supplementary Material}

\section{Proofs}
In the theorectical analysis, we fix $s_K(\bm{x},\bm{\theta})=0$. Then, we only need to consider $\mathcal{C}_{\bm{x}}\cup\mathcal{N}_{\bm{x}}=\{1,\cdots,K-1\}$. Now, the normalization factor becomes
{\small
\[
E(\bm{x},j)=1+\sum_{k'\in \mathcal{C}_{\bm{x}}}e^{s_{k'}(\bm{x},\bm{\theta})}+e^{s_j(\bm{x},\bm{\theta})}/q_{\bm{x}}(j),
\]
}\noindent
with some sampled class $j\in\mathcal{N}_{\bm{x}}$. Now, we can rewrite $R$ and $\hat{R}$ as
{\myfont
\begin{align*}
&R(\bm{\theta})=
\mathbb{E}_{\bm{x}} 
\sum_{k\in\mathcal{C}_{\bm{x}}}p(y=k|\bm{x})
\sum_{j\in \mathcal{N}_{\bm{x}}}q_{\bm{x}}(j)
\log\frac{e^{s_k(\bm{x},\bm{\theta})}}{E(\bm{x},j)}
+\sum_{k\in\mathcal{N}_{\bm{x}}}p(y=k|\bm{x})
\log\frac{e^{s_k(\bm{x},\bm{\theta})}}
{E(\bm{x},k)}
+p(y=K|\bm{x})\sum_{j\in \mathcal{N}_{\bm{x}}}q_{\bm{x}}(j)
\log\frac{1}{E(\bm{x},j)}.
\\
&\hat{R}_n(\bm{\theta})=
\frac{1}{n}\sum_{i=1}^n
\Bigg[
\sum_{k\in \mathcal{C}_{\bm{x}_i}}\mathbb{I}(y_i=k)
\sum_{j\in\mathcal{C}_{\bm{x}_i}}q_{\bm{x}_i}(j)
\log
\frac{e^{s_k(\bm{x}_i,\bm{\theta})}}
{E(\bm{x}_i,j)}
+\sum_{k\in \mathcal{N}_{\bm{x}_i}}\mathbb{I}(y_i=k)
\log
\frac
{e^{s_k(\bm{x}_i,\bm{\theta})}}
{E(\bm{x}_i,k)}
+\mathbb{I}(y_i=K)
\sum_{j\in\mathcal{C}_{\bm{x}_i}}q_{\bm{x}_i}(j)
\log\frac{1}
{E(\bm{x}_i,j)}
\Bigg].
\end{align*}
}\noindent

In the proofs, we will use point-wise notations $p_k$, $s_k$, $q_k$ and $E_k$ to represent $p(y=k|\bm{x})$, $s_k(\bm{x},\bm{\theta})$, $q_{\bm{x}}(k)$ and $E(\bm{x},k)$ for simplicity.

\subsection{Useful Lemma}

We will need the following lemma in our analysis.

\begin{lemma}\label{lem:1}
	For any norm $\|\cdot\|$ defined on the parameter space of $\bm{\theta}$, assume the quantities $\|\nabla_{\bm{\theta}}s_k\|$, $\|\nabla_{\bm{\theta}}^2s_k\|$ and $\|\nabla_{\bm{\theta}}^3s_k\|$ for $k=1,\cdots,K-1$ are bounded. Then, for any compact set $\mathbb{S}$ defined on the parameter space, we have
	{\small
	\begin{align*}
	\sup_{\bm{\theta}\in\mathbb{S}}|\hat{R}_n(\bm{\theta})-R(\bm{\theta})|\xrightarrow{p}0, \quad
	\sup_{\bm{\theta}\in\mathbb{S}}\|\nabla \hat{R}_n(\bm{\theta})-\nabla R(\bm{\theta})\|\xrightarrow{p}0, \quad
	\text{and}\quad
	\sup_{\bm{\theta}\in\mathbb{S}}\|\nabla^2 \hat{R}_n(\bm{\theta})-\nabla^2 R(\bm{\theta})\|\xrightarrow{p}0.
	\end{align*}
	}
\end{lemma}

\begin{proof}
For fixed $\bm{\theta}$, let
{\small
\begin{align*}
\psi(\bm{x},y,\bm{\theta})=&
	\sum_{k\in\mathcal{C}_{\bm{x}}}\mathbb{I}(y=k)
	\sum_{j\in\mathcal{N}_{\bm{x}}}
	q_j
	\log
	\frac{e^{s_k}}
	{1+
	\sum_{k'\in C_{\bm{x}_i}}e^{s_{k'}}+\frac{e^{s_j}}{q_j}
	}
	+\mathbb{I}(y=K)\sum_{j\in\mathcal{N}_{\bm{x}}}q_j \log\frac{1}{1+\sum_{k'\in C_{\bm{x}}}e^{s_{k'}}+\frac{e^{s_j}}{q_j}}
	\\
	&+\sum_{k\in\mathcal{N}_{\bm{x}}}\mathbb{I}(y=k)
	\log
	\frac
	{e^{s_k}}
	{1+
	\sum_{k'\in C_{\bm{x}}}e^{s_{k'}}+\frac{e^{s_k}}{q_k}
	}.
\end{align*}
}\noindent
Then we have $\hat{R}_n(\bm{\theta})=\frac{1}{n}\sum_{i=1}^n \psi(\bm{x}_i,\bm{y}_i,\bm{\theta})$ and
$R(\bm{\theta})=\mathbb{E}_{\bm{x},y}\ \psi(\bm{x},\bm{y},\bm{\theta})$. By the Law of Large Numbers, we know that $\hat{R}_n(\bm{\theta})$ converges point-wisely to $R(\bm{\theta})$ in probability.

According to the assumption, there exists a constant $M>0$ such that
{\small
\begin{align*}
\|\nabla_{\bm{\theta}} \psi(\bm{x},y,\bm{\theta})\| 
\leq\sum_{k=1}^{K-1}\left\| \nabla_{\bm{\theta}}s_k \right\| \leq M.
\end{align*}
}\noindent
Given any $\epsilon>0$, we may find a finite cover $\mathbb{S}_\epsilon\subset \mathbb{S}$ so that for any $\bm{\theta} \in \mathbb{S}$, there exists $\bm{\theta}'\in \mathbb{S}_\epsilon$ such that $|\psi(\bm{x},\bm{y},\bm{\theta}) - \psi(\bm{x},\bm{y},\bm{\theta}')| \leq M\|\bm{\theta}-\bm{\theta}'\|< \epsilon$. Since $\mathbb{S}_\epsilon$ is finite, as $n \to \infty$, $\sup_{\bm{\theta} \in \mathbb{S}_\epsilon} | \hat{R}_n(\bm{\theta})-R(\bm{\theta})|$ converges to $0$ in probability. Therefore, as $n \to \infty$, with probability $1$, we have
{\small
\[
\sup_{\bm{\theta} \in \mathbb{S}} | \hat{R}_n(\bm{\theta})-R(\bm{\theta})|< 2 \epsilon + \sup_{\bm{\theta} \in \mathbb{S}_\epsilon} |\hat{R}_n(\bm{\theta})-R(\bm{\theta})| \to 2 \epsilon.
\]
}\noindent
Let $\epsilon \to 0$, we obtain the first bound. The second and the third bounds can be similarly obtained.
\end{proof}

\subsection{Proof of Theorem \ref{theo:NP}}
\begin{proof}
$R$ can be re-written as 
{\myfont
\begin{align*}
R=\mathbb{E}_{\bm{x}}
&\sum_{j\in\mathcal{N}_{\bm{x}}}q_j
\Bigg(
\sum_{k\in\mathcal{C}_{\bm{x}}}p_k
	\log\frac{e^{s_k}}
	{1+
	\sum_{k'\in\mathcal{C}_{\bm{x}}}e^{s_{k'}}+e^{s_j}/q_j
	}
	+p_K
	\log\frac{1}
	{1+
	\sum_{k'\in\mathcal{C}_{\bm{x}}}e^{s_{k'}}+e^{s_j}/q_j
	}
	+\frac{p_j}{q_j}
	\log
	\frac{e^{s_j}}{1+\sum_{k'\in\mathcal{C}_{\bm{x}}}e^{s_{k'}}+e^{s_j}/q_j}
	\Bigg).
\end{align*}
}

For $i\in\mathcal{C}_{\bm{x}}$, we have
{\small
\begin{align*}
\nabla_{s_i} R=
&\ \mathbb{E}_{\bm{x}}
\sum_{j\in\mathcal{N}_{\bm{x}}}q_j
\Bigg[
p_i\left(1-\frac{e^{s_i}}{1+\sum_{k'\in\mathcal{C}_{x}}e^{s_{k'}}+e^{s_j}/q_j}\right)
-\sum_{k\neq i\in\mathcal{C}_{\bm{x}}}
p_k
\frac{e^{s_i}}{1+\sum_{k'\in\mathcal{C}_{x}}e^{s_{k'}}+e^{s_j}/q_j}
\\
&-p_K\frac{e^{s_i}}{1+\sum_{k'\in\mathcal{C}_{x}}e^{s_{k'}}+e^{s_j}/q_j}
-p_j/q_j\frac{e^{s_i}}{1+\sum_{k'\in\mathcal{C}_{x}}e^{s_{k'}}+e^{s_j}/q_j}
\Bigg]
\\
=&\ \mathbb{E}_{\bm{x}}
\sum_{j\in\mathcal{N}_{\bm{x}}}q_j
\Bigg[
p_i-
\left(p_K+\sum_{k\in\mathcal{C}_{x}}p_k+p_j/q_j\right)
\frac{e^{s_i}}{1+\sum_{k'\in\mathcal{C}_{x}}e^{s_{k'}}+e^{s_j}/q_j}
\Bigg].
\end{align*}
}\noindent
Similarly, for $j\in\mathcal{N}_{\bm{x}}$, we have
{\small
\begin{align*}
\nabla_{s_j} R
&=\mathbb{E}_{\bm{x}}
\ q_j\Bigg[
-\left(p_K+\sum_{k\in\mathcal{C}_{\bm{x}}}p_k\right)
\frac{e^{s_j}/q_j}{1+\sum_{k'\in\mathcal{C}_{x}}e^{s_{k'}}+e^{s_j}/q_j}
+p_j/q_j\left(1-\frac{e^{s_j}/q_j}{1+\sum_{k'\in\mathcal{C}_{x}}e^{s_{k'}}+e^{s_j}/q_j}\right)
\Bigg]
\\
&=\mathbb{E}_{\bm{x}}\ 
p_j-\left(p_K+\sum_{k\in\mathcal{C}_{\bm{x}}}p_k+p_j/q_j\right)
\frac{e^{s_j}}{1+\sum_{k'\in\mathcal{C}_{x}}e^{s_{k'}}+{e^{s_j}}/{q_j}}.
\end{align*}
}\noindent
By measuring $s_k=\log\frac{p_k}{p_K}$, we see that $\nabla_{s_k}R=0$ for $k=1,\cdots,K-1$. Therefore, $s_k=\log\frac{p_k}{p_K}$ is an extrema of $R$. Now, for $i,i'\in C_{\bm{x}}$ and $j,j'\in\mathcal{N}_{\bm{x}}$, we have
{\small
\begin{align*}
&\mathbb{H}_{ii}=\nabla_{s_is_i}^2 R=
-\mathbb{E}_{\bm{x}} \sum_{j\in\mathcal{N}_{\bm{x}}}q_j
D_j
\frac{e^{s_i}(E_j-e^{s_i})}
{E_j^2},
\\
&\mathbb{H}_{ii'}=\nabla_{s_is_{i'}}^2 R=
\mathbb{E}_{\bm{x}}
\sum_{j\in\mathcal{N}_{\bm{x}}}q_j
D_j\frac{e^{s_i}e^{s_{i'}}}{E_j^2},
\\
&\mathbb{H}_{ij}=\mathbb{H}_{ji}=\nabla_{s_is_j}^2 R=
\nabla_{s_js_i}^2 R=
\mathbb{E}_{\bm{x}}
\sum_{j\in\mathcal{N}_{\bm{x}}}
D_j\frac{e^{s_i}e^{s_j}}{E_j^2},
\\
&\mathbb{H}_{jj}=\nabla_{s_js_j}^2 R=
-\mathbb{E}_{\bm{x}}\ 
D_j\frac{e^{s_j}(E_j-e^{s_j}/q_j)}{E_j^2},
\\
&\mathbb{H}_{jj'}=\nabla_{s_js_{j'}}^2 R=0,
\end{align*}
}\noindent
where
{\small
\begin{align*}
&D_j=p_K+\sum_{k'\in\mathcal{C}_{x}}p_{k'}+p_j/q_j.
\end{align*}
}\noindent
Now, we can write
{\small
\begin{align*}
\nabla_{s}^2 R
&=
\left[
\begin{array}{ccc|ccccc}
\mathbb{H}_{i_1 i_1} & \cdots & \mathbb{H}_{i_1 i_{|\mathcal{C}_{\bm{x}}|}} & 0 & \cdots & \mathbb{H}_{i_1 j} & \cdots & 0 \\
\cdots & \cdots & \cdots & \cdots & \cdots & \cdots & \cdots & \cdots \\
\mathbb{H}_{i_{|\mathcal{C}_{\bm{x}}|}i_1} & \cdots & \mathbb{H}_{i_{|\mathcal{C}_{\bm{x}}|} i_{|\mathcal{C}_{\bm{x}}|}} & 0 & \cdots & \mathbb{H}_{i_{|\mathcal{C}_{\bm{x}}|} j} & \cdots & 0 \\
\hline
0 & \cdots & 0 & 0 & \cdots & 0 & \cdots & 0 \\
\cdots & \cdots & \cdots & \cdots & \cdots & \cdots & \cdots & \cdots \\
\mathbb{H}_{j i_1} & \cdots & \mathbb{H}_{j i_{|\mathcal{C}_{\bm{x}}|}} & 0 & \cdots & \mathbb{H}_{jj} & \cdots & 0 \\
\cdots & \cdots & \cdots & \cdots & \cdots & \cdots & \cdots & \cdots \\
0 & \cdots & 0 & 0 & \cdots & 0 & \cdots & 0 \\
\end{array}
\right]
\\
&=-
\mathbb{E}_{\bm{x}}
\sum_{j\in\mathcal{N}_{\bm{x}}}q_j\frac{D_j}{E_j}
\left[
diag(\bm{v}_j)-\frac{1}{E_j}\bm{v}_j\bm{v}_j^\top
\right].
\end{align*}
}\noindent
where $\bm{v}_j=(e^{s_{i_1}}, \cdots, e^{s_{i_{|\mathcal{C}_{\bm{x}}|}}}, 0, \cdots, e^{s_{j}}/q_{j}, \cdots, 0)^\top$. Let
{\small
\[
\bm{A}_j=
diag(\bm{v}_j)-\frac{1}{E_j}\bm{v}_j\bm{v}_j^\top.
\]
}\noindent
For any non-zero vector $\bm{\varphi}=(\varphi_1,\cdots,\varphi_{K-1})^\top\in\mathbb{R}^{K-1}$, we have
{\myfont
\begin{align*}
\bm{\varphi}^\top\bm{A}_j\bm{\varphi}&=
\sum_{i\in\mathcal{C}_{\bm{x}}}e^{s_i}\varphi_i^2+\frac{e^{s_j}}{q_j}\varphi_j^2
-\frac{1}{E_j}\left(\sum_{i\in\mathcal{C}_{\bm{x}}}e^{s_i}\varphi_i+\frac{e^{s_j}}{q_j}\varphi_j\right)^2
\geq \frac{\left(\sum_{i\in\mathcal{C}_{\bm{x}}}e^{s_i}\varphi_i+\frac{e^{s_j}}{q_j}\varphi_j\right)^2}{\sum_{i\in\mathcal{C}_{\bm{x}}}e^{s_i}+\frac{e^{s_j}}{q_j}}-\frac{1}{E_j}\left(\sum_{i\in\mathcal{C}_{\bm{x}}}e^{s_i}\varphi_i+\frac{e^{s_j}}{q_j}\varphi_j\right)^2
>0,
\end{align*}
}\noindent
for every $j\in\mathcal{N}_{\bm{x}}$, where the first inequality is by the Cauchy-Schwarz inequality and the second inequality is because $0<\sum_{i\in\mathcal{C}_{\bm{x}}}e^{s_i}+\frac{e^{s_j}}{q_j}<E_j$. Therefore, $-\nabla_{s}^2 R=\mathbb{E}_{\bm{x}}\sum_{j\in\mathcal{N}_{\bm{x}}}q_j\frac{D_j}{E_j}\bm{A}_j$ is positive-definite and $R$ is strongly concave with respect to $s$. Hence, $s_k=\log \frac{p_k}{p_K}$ for $k=1,\cdots,K-1$ is the only maxima of $R$.
\end{proof}

\subsection{Proof of Theorem \ref{theo:con}}
\begin{proof}
$R$ can be re-written as 
{\small
\begin{align*}
R(\bm{\theta})=
\mathbb{E}_{\bm{x}} 
\sum_{k\in\mathcal{C}_{\bm{x}}}p_k
\sum_{j\in \mathcal{N}_{\bm{x}}}q_j
\log\frac{e^{s_k}}{E_j}
+\sum_{k\in\mathcal{N}_{\bm{x}}}p_k
\log\frac{e^{s_k}}
{E_k}
+p_K\sum_{j\in \mathcal{N}_{\bm{x}}}q_j
\log\frac{1}{E_j}.
\end{align*}
}
Note that $E_j$ for any $j$ can be viewed as a function of $\bm{s}=(s_1,\cdots,s_{K-1})^{\top}$. Define the following function 
{\small
\[
G(\bm{s})=\sum_{k\in\mathcal{C}_{\bm{x}}}p_k
\sum_{j\in \mathcal{N}_{\bm{x}}}q_j
\log E_j
+\sum_{k\in\mathcal{N}_{\bm{x}}}p_k
\log E_k
+p_K\sum_{j\in \mathcal{N}_{\bm{x}}}q_j
\log E_j,
\]
}
then for any $\bm{\theta}\neq\bm{\theta}^*$,
{\small
\begin{align*}
R(\bm{\theta}^*)-R(\bm{\theta})
&=\mathbb{E}_{\bm{x}} 
\sum_{k\in\mathcal{C}_{\bm{x}}}p_k
\sum_{j\in \mathcal{N}_{\bm{x}}}q_j
\left(\log\frac{E_j}{E_j^*}+s_k^*-s_k\right)
+\sum_{k\in\mathcal{N}_{\bm{x}}}p_k
\left(\log\frac{E_k}{E_k^*}+s_k^*-s_k\right)
+p_K\sum_{j\in \mathcal{N}_{\bm{x}}}q_j
\log\frac{E_j}{E_j^*}
\\
&=\mathbb{E}_{\bm{x}} 
\sum_{k\in\mathcal{C}_{\bm{x}}}p_k
\sum_{j\in \mathcal{N}_{\bm{x}}}q_j
\log\frac{E_j}{E_j^*}
+\sum_{k\in\mathcal{N}_{\bm{x}}}p_k
\log\frac{E_k}{E_k^*}
+p_K\sum_{j\in \mathcal{N}_{\bm{x}}}q_j
\log\frac{E_j}{E_j^*}
+\sum_{k=1}^{K-1}p_k(s_k^*-s_k)
\\
&=G(\bm{s})-G(\bm{s}^*)-\nabla G(\bm{s}^*)^{\top}(\bm{s}-\bm{s}^*)=\Delta(\bm{s},\bm{s}^*),
\end{align*}
}\noindent
where $\Delta(\bm{s},\bm{s}^*)$ is the Bregman divergence of the convex function $G(\bm{s})$. Since $G(\cdot)$ is convex, we have $\Delta(s,s^*)\geq0$ and $\Delta(\bm{s},\bm{s}^*)=0$ only when $\bm{s}=\bm{s}^*$. Under the assumption that the parameter space is compact and $\forall\bm{\theta}\neq\bm{\theta}^*$ we have $\mathbb{P}_{\mathcal{X}}\left(s_k(\bm{x},\bm{\theta})\neq s_k(\bm{x},\bm{\theta}^*)\right)>0$ for $k\neq K$, we know that $R(\bm{\theta})<R(\bm{\theta}^*)$ for any $\bm{\theta}\neq\bm{\theta}^*$. 

Given any $\varepsilon'>0$, there exists $\varepsilon>0$ that $R(\bm{\theta}^*)-R(\bm{\theta})<\varepsilon$ implies $\|\bm{\theta}^*-\bm{\theta}\|<\varepsilon'$. Now according to Lemma \ref{lem:1}, there exists a $\delta>0$, when $n\to\infty$, we have
{\small
\begin{align*}
R(\bm{\theta}^*)-R(\hat{\bm{\theta}})
=R(\bm{\theta}^*)-\hat{R}_n(\bm{\theta}^*)+\hat{R}_n(\bm{\theta}^*)-R(\hat{\bm{\theta}})
&\leq R(\bm{\theta}^*)-\hat{R}_n(\bm{\theta}^*)+\hat{R}_n(\hat{\bm{\theta}})-R(\hat{\bm{\theta}})
\\
&\leq |R(\bm{\theta}^*)-\hat{R}_n(\bm{\theta}^*)|+|\hat{R}_n(\hat{\bm{\theta}})-R(\hat{\bm{\theta}})|
<2\delta.
\end{align*}
}\noindent
This implies that $\|\hat{\bm{\theta}}-\bm{\theta}^*\|<\delta'$ for any $\delta'>0$.
\end{proof}

\subsection{Proof of Theorem \ref{theo:norm}}

\begin{proof}
By the Mean Value Theorem, we have
{\small
\begin{align}
\sqrt{n}(\hat{\bm{\theta}}-\bm{\theta}^*)=-\nabla^2\hat{R}_n(\bar{\bm{\theta}})^{-1}\sqrt{n}\nabla\hat{R}_n(\bm{\theta}^*) ,
\end{align}
}\noindent
where $\bar{\bm{\theta}}=t\bm{\theta}^*+(1-t)\hat{\bm{\theta}}$ for some $t\in[0,1]$. Note that Lemma \ref{lem:1} implies that $\nabla^2 \hat{R}_n(\bar{\bm{\theta}})^{-1}$ converges to $\nabla^2 R(\bar{\bm{\theta}})^{-1}$ in probability; moreover, $\hat{\bm{\theta}}\rightarrow \bm{\theta}^*$ in probability and hence $\bar{\bm{\theta}}\rightarrow \bm{\theta}^*$ in probability. By the Slutsky's Theorem, the limit distribution of $\sqrt{n}(\hat{\bm{\theta}}-\bm{\theta}^*)$ is given by
{\small
\[
-\nabla^2 R(\bm{\theta}^*)^{-1}\sqrt{n}\nabla\hat{R}_n(\bm{\theta}^*) .
\]
}\noindent
Observe that $\sqrt{n}\nabla\hat{R}_n(\bm{\theta}^*)$ is the sum of $n$ i.i.d. random vectors with mean $\mathbb{E}\sqrt{n}\nabla\hat{R}_n(\bm{\theta}^*)=\sqrt{n}\mathbb{E}\nabla R(\bm{\theta}^*)=0$, and the variance of $\sqrt{n}(\hat{\bm{\theta}}-\bm{\theta}^*)$ is
{\small
\begin{align*}
Var\left(\sqrt{n}(\hat{\bm{\theta}}-\bm{\theta}^*)\right)=\nabla^2 R(\bm{\theta}^*)^{-1}Var\left(\sqrt{n}\nabla\hat{R}_n(\bm{\theta}^*)\right)\nabla^2 R(\bm{\theta}^*)^{-1}.
\end{align*}
}\noindent
From the proof of Theorem 1, we have
{\small
\begin{align}
\nabla^2R(\bm{\theta})&=-\mathbb{E}_{\bm{x}}
\bm{\nabla}\left[
\sum_{j\in\mathcal{N}_{\bm{x}}}q_j
\frac{D_j}{E_j}\bm{A}_j
\right]
\bm{\nabla}^\top,
\label{eq:H}
\end{align}
}\noindent
where
{\small
\[
\bm{\nabla}=
diag\left(
\left(
\nabla_{i_1},\cdots,\nabla_{i_{|\mathcal{C}_{\bm{x}}|}},\nabla_{j_1},\cdots,
\nabla_{j_{|\mathcal{N}_{\bm{x}}|}}\right)^\top
\right) 
\]
}\noindent
and $\nabla_{k}=\nabla_{\bm{\theta}}s_k$.

Measuring $\nabla^2R(\bm{\theta})$ at $\bm{\theta}^*$, we have
{\small
\begin{align}
\nabla^2R(\bm{\theta}^*)
&=-\mathbb{E}_{\bm{x}}
\bm{\nabla}\bm{M}\bm{\nabla}^\top
\label{eq:Hstar}
\end{align}
}\noindent
where
{\small
\[
\bm{M}=\sum_{j\in\mathcal{N}_{\bm{x}}}q_j
\left[diag(\bm{u}_j)-\frac{1}{D_j}\bm{u}_j\bm{u}_j^\top\right],
\]
}\noindent
where $\bm{u}_j=(p_{i_1}, \cdots, p_{i_{|\mathcal{C}_{\bm{x}}|}}, 0, \cdots, p_{j}/q_{j}, \cdots, 0)^\top$. 
By following the proof of Theorem 1, it is easy to show that $\bm{M}\succ0$ is positive definite. 

Next, we derive $Var\left(\sqrt{n}\nabla\hat{R}_n(\bm{\theta}^*)\right)$. Introduce some Bernoulli variables $Q_j$ for $j\in\mathcal{N}_{\bm{x}}$ with $p(Q_j=1|\bm{x})=q_j$. Now, for $i,i'\in C_{\bm{x}}$ and $j,j'\in\mathcal{N}_{\bm{x}}$, we have
{\small
\begin{flalign*}
\mathbb{V}_{ii}
&=Var\left(\nabla_{i}\hat{R}_n(\bm{\theta}^*), \nabla_{i}\hat{R}_n(\bm{\theta}^*) \right)
\\
&=\mathbb{E}_{\bm{x},Q}\ 
Q
\Bigg[
p_i\left(1-\frac{e^{s_i^*}}{1+\sum_{k'\in\mathcal{C}_{\bm{x}}}e^{s_{k'}^*}+e^{s_j^*}/q_j}\right)^2
+(D_j-p_i)\left(
\frac{e^{s_i^*}}{1+\sum_{k'\in\mathcal{C}_{\bm{x}}}e^{s_{k'}^*}+e^{s_j^*}/q_j}
\right)^2
\Bigg]\cdot \nabla_i \nabla_i^\top
\\
&=\mathbb{E}_{\bm{x}}\ 
\sum_{j\in\mathcal{N}_{\bm{x}}}q_j
\frac{p_i(D_j-p_i)}{D_j}
\cdot \nabla_i \nabla_i^\top,
\\
\\
\mathbb{V}_{ii'}
&=Var\left(\nabla_{i}\hat{R}_n(\bm{\theta}^*), \nabla_{i'}\hat{R}_n(\bm{\theta}^*) \right)
=\mathbb{E}_{\bm{x},Q} \ 
Q
\Bigg[
(D_j-p_i-p_{i'})\frac{p_ip_{i'}}{D_j^2}
-p_i(1-\frac{p_i}{D_j})\frac{p_{i'}}{D_j}
-p_{i'}(1-\frac{p_{i'}}{D_j})\frac{p_i}{D_j}
\Bigg]
\cdot \nabla_i \nabla_{i'}^\top
\\
&=-\mathbb{E}_{\bm{x}}\ 
\sum_{j\in\mathcal{N}_{\bm{x}}}q_j
\frac{p_ip_{i'}}{D_j}
\cdot \nabla_i \nabla_{i'}^\top.
\\
\\
\mathbb{V}_{jj}
&=Var\left(\nabla_{j}\hat{R}_n(\bm{\theta}^*), \nabla_{j}\hat{R}_n(\bm{\theta}^*) \right)
=\mathbb{E}_{\bm{x},Q} \ 
Q\Bigg[
\frac{p_j}{q_j}
\left(1-\frac{p_j/q_j}{D_j}\right)^2
+\left(D_j-p_j/q_j\right)\frac{p_j^2/q_j^2}{D_j^2}
\Bigg]
\cdot \nabla_j \nabla_j^\top
\\
&=\mathbb{E}_{\bm{x}}
\sum_{j\in\mathcal{N}_{\bm{x}}}
\frac{p_j\left(D_j-p_j/q_j\right)}{D_j}
\cdot \nabla_j \nabla_j^\top.
\\
\\
\mathbb{V}_{jj'}&=\bm{0}.
\\
\\
\mathbb{V}_{ij}
&=\mathbb{V}_{ji}
=Var\left(\nabla_{i}\hat{R}_n(\bm{\Theta}^*), \nabla_{j}\hat{R}_n(\bm{\Theta}^*) \right)
\\
&=\mathbb{E}_{\bm{x},\bm{Q}} \ 
Q
\Bigg[
\left(D_j-p_i-p_j/q_j\right)\frac{p_ip_j/q_j}{D_j^2}
-p_i\left(1-\frac{p_i}{D_j}\right)\frac{p_j/q_j}{D_j}
-p_j/q_j\left(1-\frac{p_j/q_j}{D_j}\right)\frac{p_i}{D_j}
\Bigg]
\cdot \nabla_i \nabla_{i'}^\top
\\
&=-\mathbb{E}_{\bm{x}}\ 
\sum_{j\in\mathcal{N}_{\bm{x}}}
\frac{p_ip_j}{D_j}
\cdot \nabla_i \nabla_{i'}^\top.
\end{flalign*}
}

Now, the variance can be written as
{\small
\begin{align*}
V(\bm{\theta}^*)
&=Var\left(\sqrt{n}\nabla\hat{R}_n(\bm{\theta}^*)\right)
\\
&=
\left[
\begin{array}{ccc|ccccc}
\mathbb{V}_{i_1 i_1} & \cdots & \mathbb{V}_{i_1 i_{|\mathcal{C}_{\bm{x}}|}} & 0 & \cdots & \mathbb{V}_{i_1 j} & \cdots & 0 \\
\cdots & \cdots & \cdots & \cdots & \cdots & \cdots & \cdots & \cdots \\
\mathbb{V}_{i_{|\mathcal{C}_{\bm{x}}|}i_1} & \cdots & \mathbb{V}_{i_{|\mathcal{C}_{\bm{x}}|} i_{|\mathcal{C}_{\bm{x}}|}} & 0 & \cdots & \mathbb{V}_{i_{|\mathcal{C}_{\bm{x}}|} j} & \cdots & 0 \\
\hline
0 & \cdots & 0 & 0 & \cdots & 0 & \cdots & 0 \\
\cdots & \cdots & \cdots & \cdots & \cdots & \cdots & \cdots & \cdots \\
\mathbb{V}_{j i_1} & \cdots & \mathbb{V}_{j i_{|\mathcal{C}_{\bm{x}}|}} & 0 & \cdots & \mathbb{V}_{jj} & \cdots & 0 \\
\cdots & \cdots & \cdots & \cdots & \cdots & \cdots & \cdots & \cdots \\
0 & \cdots & 0 & 0 & \cdots & 0 & \cdots & 0 \\
\end{array}
\right].
\end{align*}
}\noindent
By comparing $\nabla^2 R(\bm{\theta}^*)$ and $V(\bm{\theta}^*)$, we immediately have $-\nabla^2 R(\bm{\theta}^*)=V(\bm{\theta}^*)$ and hence
{\small
\begin{align*}
	Var\left(\sqrt{n}(\hat{\bm{\theta}}-\bm{\theta}^*)\right)=\left[\mathbb{E}_{\bm{x}}\bm{\nabla}\bm{M}\bm{\nabla}^\top\right]^{-1}.
\end{align*}
}
\end{proof}

\subsection{Proof of Corollary \ref{theo:IV}}
\begin{proof}
By following the proof of Theorem \ref{theo:norm}, it is easy to show that the statistical variance of the softmax logistic regression in Eq. (\ref{eq:softmax}) is $[\mathbb{E}_{\bm{x}}\bm{\nabla}\bm{M}^{mle}\bm{\nabla}^\top]^{-1}$ (with $s_K=0$ fixed), where
{\small
\[
\bm{M}^{mle}=
diag\left(
\left[
\begin{array}{c}
p_1 \\
\vdots \\
p_{K-1} \\
\end{array}
\right]
\right)
-
\left[
\begin{array}{c}
p_1 \\
\vdots \\
p_{K-1} \\\end{array}
\right]
\left[
\begin{array}{c}
p_1 \\
\vdots \\
p_{K-1} \\\end{array}
\right]^\top.
\]
}\noindent
When $\sum_{k\in\mathcal{C}_{\bm{x}}\cup\{K\}}p(k,\bm{x})\rightarrow1$, we have $\sum_{j'\in\mathcal{N}_{\bm{x}}}p_{j'}\rightarrow0$ and $D_j\rightarrow1$. Then,
{\myfont
\begin{align*}
\bm{M}=
diag\left(
\left[
\begin{array}{c}
p_{i_1} \\
\vdots \\
p_{i_{|\mathcal{C}_{\bm{x}}|}} \\
p_{j_1} \\
\vdots \\
p_{j_{|\mathcal{N}_{\bm{x}}|}} \\
\end{array}
\right]
\right)
-
\left[
\begin{array}{ccc|ccc}
p_{i_1}p_{i_1} & \cdots & p_{i_1}p_{i_{|\mathcal{C}_{\bm{x}}|}} & 
p_{i_1}\sum_{j'\in\mathcal{N}_{\bm{x}}}p_{j'} & \cdots & p_{i_1}\sum_{j'\in\mathcal{N}_{\bm{x}}}p_{j'} \\
\cdots & \cdots & \cdots & \cdots & \cdots & \cdots \\
p_{i_{|\mathcal{C}_{\bm{x}}|}}p_{i_1} & \cdots & p_{i_{|\mathcal{C}_{\bm{x}}|}}p_{i_{|\mathcal{C}_{\bm{x}}|}} & p_{i_{|\mathcal{C}_{\bm{x}}|}}\sum_{j'\in\mathcal{N}_{\bm{x}}}p_{j'} & \cdots & p_{i_{|\mathcal{C}_{\bm{x}}|}}\sum_{j'\in\mathcal{N}_{\bm{x}}}p_{j'} \\
\hline
p_{i_1}\sum_{j'\in\mathcal{N}_{\bm{x}}}p_{j'} & \cdots & p_{i_{|\mathcal{C}_{\bm{x}}|}}\sum_{j'\in\mathcal{N}_{\bm{x}}}p_{j'} & p_{j_1}^2/q_{j_1} & \cdots & 0 \\
\cdots & \cdots & \cdots & \cdots & \cdots & \cdots \\
p_{i_1}\sum_{j'\in\mathcal{N}_{\bm{x}}}p_{j'} & \cdots & p_{i_{|\mathcal{C}_{\bm{x}}|}}\sum_{j'\in\mathcal{N}_{\bm{x}}}p_{j'} & 0 & \cdots & p_{j_{|\mathcal{N}_{\bm{x}}|}}^2/q_{j_{|\mathcal{N}_{\bm{x}}|}} \\
\end{array}
\right].
\end{align*}
}\noindent
If we arrange the index order in $\bm{M}^{mle}$ according to the index order in $\bm{M}$ and denote $\bm{\Delta}=\bm{M}-\bm{M}^{mle}$, we have
{\small
\[
\bm{\Delta}=\left[
\begin{array}{cc}
\bm{\Delta}_1 & \bm{\Delta}_2 \\
\bm{\Delta}_2^\top & \bm{\Delta}_3 \\
\end{array}
\right]\rightarrow\bm{0},
\]
}
because
{\small
\begin{align*}
&\bm{\Delta}_1=\bm{0},
\\
&
\bm{\Delta}_2=
\left[
\begin{array}{ccc} 
p_{i_1}(p_{j_1}-\sum_{j'\in\mathcal{N}_{\bm{x}}}p_{j'}) & \cdots & p_{i_1}(p_{j_{|\mathcal{N}_{\bm{x}}|}}-\sum_{j'\in\mathcal{N}_{\bm{x}}}p_{j'}) \\
\cdots & \cdots & \cdots \\
p_{i_{|\mathcal{C}_{\bm{x}}|}}(p_{j_1}-\sum_{j'\in\mathcal{N}_{\bm{x}}}p_{j'}) & \cdots & p_{i_{|\mathcal{C}_{\bm{x}}|}}(p_{j_{|\mathcal{N}_{\bm{x}}|}}-\sum_{j'\in\mathcal{N}_{\bm{x}}}p_{j'}) \\
\end{array}
\right]\rightarrow\bm{0},
\\
&
\bm{\Delta}_3=
\left[
\begin{array}{ccc}
p_{j_1}^2(1-1/q_{j_1}) & \cdots & p_{j_1}p_{j_{|\mathcal{N}_{\bm{x}}|}} \\
\cdots & \cdots & \cdots \\
p_{j_{|\mathcal{N}_{\bm{x}}|}}p_{j_1} & \cdots & p_{j_{|\mathcal{N}_{\bm{x}}|}}^2(1-1/q_{j_{|\mathcal{N}_{\bm{x}}|}}) \\
\end{array}
\right]\rightarrow\bm{0}.
\end{align*}
}\noindent
This completes the proof.
\end{proof}

\section{The Beam Search Algorithm}
The beam search algorithm used in both training and testing is depicted in Algorithm \ref{algo:beamsearch}.
\begin{algorithm}[h]
\small
\caption{The Beam Search Algorithm.}
\label{algo:beamsearch}
\begin{algorithmic}[1]
\STATE \textbf{Input:} The root of the tree, input data point $\bm{x}$ and Beam width $J$.
\STATE \textbf{Output:} The $J$ candidate classes.
\vskip 0.1in

\STATE Initialize stack $\mathcal{S}\leftarrow root$ and stack $\mathcal{S}'\leftarrow \emptyset$;
\STATE Initialize the candidate class set $\mathcal{E}\leftarrow \emptyset$;
\WHILE {true}
\IF {$\mathcal{S}$ is empty}
\STATE Break;
\ENDIF
\FOR {$i=1$ to $\mathcal{S}.size()$}
\IF {$\mathcal{S}_i$ is a leaf}
\STATE $\mathcal{E}.pushback(\mathcal{S}_i)$;
\ELSE
\FOR {$c=1$ to $\mathcal{S}_i.Child.size()$}
\STATE Accumulate the score to $\mathcal{S}_i.Child(c)$;
\STATE $\mathcal{S}'.pushback(\mathcal{S}_i.Child(c))$;
\ENDFOR
\ENDIF
\ENDFOR
\STATE $\mathcal{S}.clear()$;
\IF {$\mathcal{S}'.size()>J$}
\STATE // \emph{Using the max heap.} 
\STATE Find the top-$J$ nodes with the highest accumulated scores in $\mathcal{S}'$ and push them into $\mathcal{S}$;
\ELSE
\STATE $\mathcal{S}\leftarrow \mathcal{S}'$;
\ENDIF
\STATE $\mathcal{S}'.clear()$;
\ENDWHILE
\STATE // \emph{Using the max heap.} 
\STATE Return the top-$J$ classes with the highest scores in $\mathcal{E}$;
\end{algorithmic}
\end{algorithm}

\section{A Hierarchical Clustering Method for Generating the Tree Structure}
 
Given the data points of a dataset, we can obtain the center, i.e., the average data point, of each class by scanning the data once and get $\bar{\bm{X}}\in\mathbb{R}^{K\times d}$, where $K$ is the number of classes and $d$ is the feature dimension. Then, a hierarchical clustering algorithm in Algorithm \ref{algo:hct} is performed by viewing each row of $\bar{\bm{X}}$ as a separate data point. In Algorithm \ref{algo:hct}, the function `Split(root)' in step 16 has already constructed a $b$-nary tree, which can be used by the Beam Tree Algorithm. However, the clustering algorithm, e.g., the $k$-means algorithm, may generate imbalanced clusters in step 9, and the resulting $b$-nary tree in step 16 may be imbalanced and affect the efficiency of Beam Tree. A simple way to fix this problem is to fetch the labels (leaves) in the tree in step 16 from left to right, where the obtained label order maintains a rough similarity relationship among the classes. We then assign the ordered labels to the leaves of a new balanced $b$-nary tree from left to right.
\begin{algorithm}[!ht]
\small
\caption{A Hierarchical Clustering Algorithm for Generating the Tree over Class Labels.}
\label{algo:hct}
\begin{algorithmic}[1]
\STATE \textbf{Input:} $K$, $b$ and $\bar{\bm{X}}$.
\STATE \textbf{Output:} a $b$-nary tree.
\vskip 0.1in

\STATE \textbf{Function} Split(node $o$)
\WHILE {true}
\IF {$o$ is assigned with only one label}
\STATE $o.isleaf=true$;
\STATE Return;
\ENDIF
\STATE Perform any clustering algorithm, e.g., k-means, on the labels associated with the node $o$ and obtain $b$ clusters $\{\mathcal{L}_1,\cdots,\mathcal{L}_b\}$;
\STATE Split $o$ into $b$ children $\{o_1,\cdots,o_b\}$ and assign the label clusters $\{\mathcal{L}_1,\cdots,\mathcal{L}_b\}$ to them respectively;
\FOR {$i=1$ to $b$}
\STATE Split($o_i$);
\ENDFOR
\ENDWHILE
\vskip 0.1in
\STATE Assign root with all labels $\{1,2,\cdots,K\}$;
\STATE Split(root);
\STATE Get the label order in the leaves from left to right;
\STATE Assign the labels to the leaves of a new balanced $b$-nary tree from left to right;

\STATE Return the balanced $b$-nary tree;
\end{algorithmic}
\end{algorithm}

\section{Experimental Details}
\begin{figure*}[!ht]
\centering
\subfigure{
\includegraphics[scale=0.6]{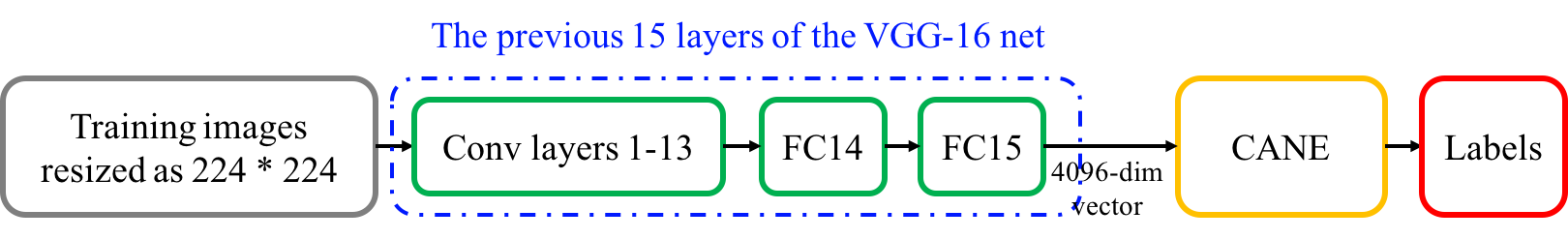}
}
\caption{\small The neural network structure used for the ImageNet datasets. `FC' indicates fully-connected layer.}
\label{fig:vgg}
\end{figure*}
Hyper-parameter tuning is computationally expensive. In order to efficiently select a good setting of the hyper-parameters, we let each method process half epoch of the training data and use another 10\% held-out subset of the training set to tune hyper-parameters. For every classifier, the learning rate $\eta$ needs to be tuned. For the LOMTree method, by following \cite{choromanska2015logarithmic}, we choose the number of the internal nodes in its binary tree from a set $\{K-1, 4K-1, 16K-1, 64K-1\}$, and tune the swap resistance from $\{4, 16, 64, 256\}$. The Recall Tree method has a default setting for large class problem in \cite{daume2016logarithmic}, which is also adopted in the experiments.

The VGG-16 network structure used in ImageNet-2010 and ImageNet-10K datasets is provided in Fig. \ref{fig:vgg}. Parameters of Conv layers 1-13, FC14 and FC15 are pre-trained on the ImageNet 2012 dataset.

\end{document}